\documentclass{article} 
\usepackage{iclr2027_conference,times}

\usepackage{hyperref}
\usepackage{url}

\usepackage{amsfonts} 
\usepackage{graphicx}
\usepackage{subcaption}
\usepackage{amsmath}
\usepackage{amssymb}
\usepackage{amsthm}

\usepackage{algorithm}
\usepackage{algorithmic}
\newcommand{\argmin}{\mathop{\rm arg~min}\limits}

\theoremstyle{plain}
\newtheorem{theorem}{Theorem}[section]
\newtheorem{proposition}[theorem]{Proposition}
\newtheorem{lemma}[theorem]{Lemma}

\theoremstyle{definition}

\theoremstyle{remark}

\title{AUC Maximization from Biased \\ Positive-unlabeled Data with Confidence}

\author{
\textbf{Atsutoshi Kumagai},
\textbf{Tomoharu Iwata},
\textbf{Hiroshi Takahashi},
\textbf{Taishi Nishiyama}, \\
\textbf{Kazuki Adachi},
\textbf{Yasuhiro Fujiwara} \\
NTT, Inc. \\
\texttt{atsutoshi.kumagai@ntt.com}
}

\iclrfinalcopy 
\begin{document}

\maketitle

\lhead{Preprint}

\begin{abstract}
Maximizing the area under the receiver operating characteristic curve (AUC) is a standard approach to imbalanced binary classification.
Although positive and negative data are required for maximizing the AUC, negative data are often difficult to collect in some real-world applications due to privacy concerns or the need for specialized expertise to annotate them.
Thus, AUC maximization from positive and unlabeled (PU) data has been attracting attention.
Existing methods assume that labeled positive data are unbiased samples from the true positive distribution.
However, this ideal assumption is often violated in practice.
In this paper, we propose a method to maximize the AUC from biased PU data.
To address the bias, our key idea is to exploit {\it confidence}, i.e., the probability that an instance is positive, associated with the small number of labeled positive data.
We derive an estimator of the AUC risk using biased PU data with confidence, enabling AUC maximization under such bias. We further show that the rewritten AUC risk induces a Bayes-optimal AUC ranking even when the available confidence is any strictly increasing transformation of the true posterior probability.
We experimentally show the effectiveness of our method on eight real-world datasets.
\end{abstract}

\section{Introduction}
\label{intro}

In many real-world binary classification tasks
such as cyber security
\citep{mirsky2018kitsune,bagui2021resampling},
medical care
\citep{yang2021automated},
product inspection
\citep{park2016machine},
and fraud detection
\citep{su2021positive},
{\it class-imbalance} frequently arises, where the amount of positive data is much smaller than that of negative data
\citep{johnson2019survey}.
For such imbalanced data, classification accuracy, which is the standard performance metric in ordinary classification, is not a suitable measure
\citep{yang2022auc}.
Instead, the area under the receiver operating characteristic curve (AUC) is widely used
\citep{bradley1997use,mcdermott2024closer}.
The AUC represents the probability that a classifier will rank a positive instance higher than a negative one
\citep{yang2022auc}. 
Owing to the nature of the ranking, the AUC can adequately measure the classifier's performance even with imbalanced data.
Consequently, maximizing the AUC enables accurate classifiers to be learned from imbalanced data
\citep{liu2020stochastic,yuan2021compositional,yang2022auc}.

\begin{figure*}[t]
  \centering
  \includegraphics[width=13.0cm]{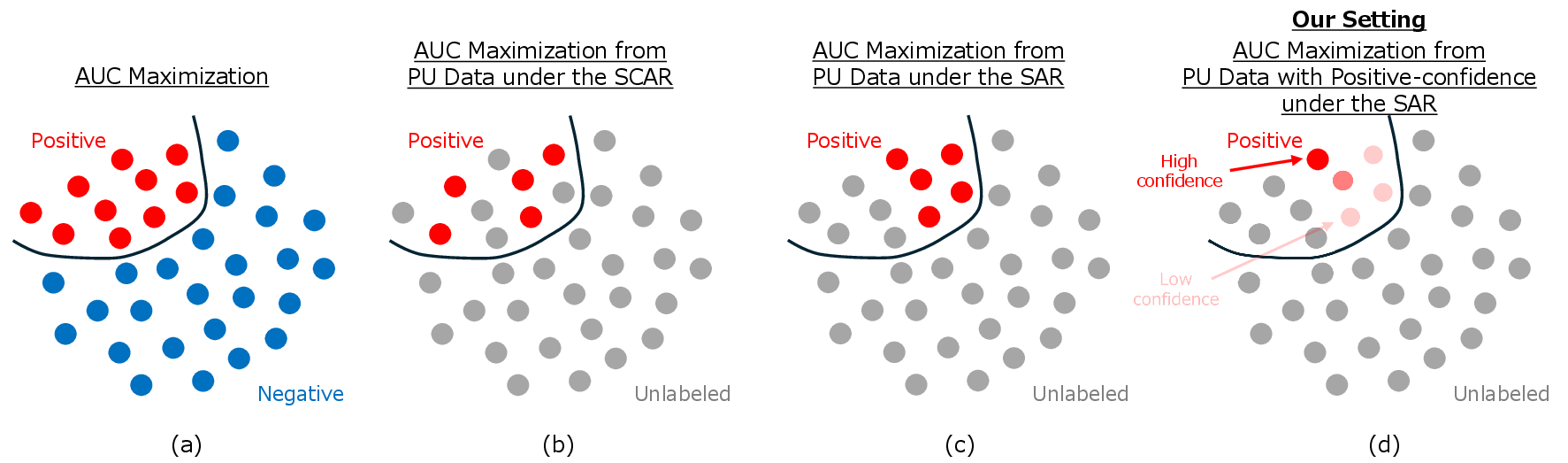}
  \caption{Illustrations of our problem setting and other related settings. Red, blue, and gray points are positive, negative, and unlabeled data, respectively. The dark/light red colors show high/low confidence values for positive data. 
  The black line is illustrated to visually distinguish positive and negative data.
  (a) Standard AUC maximization with PN data. (b) AUC Maximization from PU data under the SCAR setting where labeled positive data are unbiased. (c) AUC Maximization from PU data under the SAR setting where labeled positive data are biased. Note that there are no studies for this setting. (d) Our problem setting: AUC maximization from PU data with positive-confidence under the SAR setting. 
  In class-imbalanced settings, the number of labeled positive instances (each equipped with confidence) is small.
  }
  \label{overview}
\end{figure*}

Although labeled positive and negative (PN) data are required for AUC maximization,
negative data are often difficult to collect in practice.
For example, in cyber security, malicious data (positive data) can be collected from public blocklists, but benign data of legitimate users (negative data) are often unavailable due to privacy concerns, and 
identifying benign data from given unlabeled data requires a high level of expertise
\citep{mirsky2018kitsune,kumagaiauc,kumagaipositive}.
In medical care, although positive patient data can be collected on the basis of confirmed diagnoses, 
treating all other data as negative is problematic since
such data often include false-negatives or pre-symptomatic patients who have already contracted the disease~\citep{bekker2020learning}.
Similar situations occur in other applications such as product inspection
\citep{kumagaiauc} 
and fraud detection
\citep{su2021positive}.

To address such situations, several studies have proposed methods to maximize the AUC from
only positive and unlabeled (PU) data
\citep{sakai2018semi,xie2018semi,xie2024weakly}.
These methods assume that labeled positive data are unbiased samples from the true positive distribution, which is known as the
{\it selected completely at random} (SCAR) setting
\citep{elkan2008learning}.
Although this ideal assumption enables an unbiased estimator of the AUC to be derived with PU data,
it is often violated in practice since labeled positive data are often {\it biased}.
For example, blocklists tend to contain certain (widely known) types of malicious data, and
diseases that can be diagnosed are usually limited to well-understood cases.

To cope with such biased situations, some PU learning studies for ordinary classification have recently considered the {\it selected at random} (SAR) setting, where the probability of selecting positive data to be labeled depends on their feature values
\citep{kato2018learning,teisseyre2025learning,gerych2022recovering}.
While this setting is more realistic, these methods are not designed for maximizing the AUC.
In addition, to make the problem tractable, they require strong assumptions about data distributions or labeling mechanisms, such as the PN distributions being separated
\citep{gerych2022recovering}
or the selection probability of positive data following some restricted function forms
\citep{gerych2022recovering,he2018instance,teisseyre2025learning}.
When the assumption is violated, the performance of these methods drastically deteriorates.

To address this problem, we propose a method for maximizing the AUC under the SAR setting by using biased positive data equipped with {\it confidence} and unlabeled data, without such restrictive assumptions about data distributions or labeling mechanisms.
Here, the confidence is formally defined as the probability that a given instance is positive,
and we refer to this confidence for labeled positive data as positive-confidence.
Figure \ref{overview} illustrates our problem setting.

Such confidence need not be collected specifically for the proposed method; 
in many applications, it is already produced as a by-product of existing labeling or operational processes.
For example, when the same instance is labeled by multiple annotators to enhance reliability,
the average of these labels can naturally form confidence
\citep{ishidaperformance}.
Such multi-annotator labeling is used in various domains such as medical care, cyber security, and crowdsourcing
\citep{le2023learning,kantchelian2015better,tang2025confidence}.
Confidence-like scores are also routinely available in applications,
such as BI-RADS scores representing the likelihood of cancer in medical care
\citep{sato2025}, 
alert severity levels in security operations
\citep{jalalvand2024alert}, 
and risk scores in financial fraud detection. 
Moreover, automatic labeling systems, including LLM-based annotators and pre-trained classifiers, can also provide such information
\citep{ishida2018binary,tang2025confidence}. 
These examples motivate the growing body of confidence-based weakly supervised learning methods
\citep{ishida2018binary,cao2021learning,berthon2021confidence,feng2021pointwise,wang2023binary}.

However, to the best of our knowledge,
positive-confidence has not been used for PU learning under the SAR setting.
The key idea is that positive-confidence provides information that allows us to correct the bias in the labeled positive data.
By using such confidence, 
we derive an estimator of the AUC risk from biased PU data without relying on the strong assumptions about data distributions or labeling mechanisms required by existing SAR methods.
This enables us to maximize the AUC from biased PU data.
A practical concern is that obtained confidence may not be accurately calibrated as a probability.
We further show that the proposed method does not require exact calibration to recover the optimal AUC ranking: when the observed confidence is given by any strictly increasing transformation of the true posterior probability\footnote{We say that a function $h: [0,1] \to [0,1] $ is strictly increasing if it satisfies $h(a)<h(b)$ whenever $a<b$.}, 
minimizing the resulting AUC risk still yields the Bayes-optimal AUC ranking.
This indicates that the proposed method can also exploit confidence that is systematically over- or under-confident, as long as its ordering is preserved.
In addition, our method is model-agnostic: it can use any differentiable classifier such as linear classifier,
kernel-based classifier, and neural network, which is preferable in practice.

\section{Related Work}

Various methods for AUC maximization have been proposed
\citep{brefeld2005auc,ying2016stochastic,liu2020stochastic,yuan2021compositional,yang2022auc}.
Previous studies
\citep{fujino2016semi,yuan2021compositional,yuan2021large,wang2023novel}
have shown that AUC maximization methods often outperform other methods for imbalanced data
such as class balanced loss
\citep{charoenphakdee2019symmetric},
focal loss
\citep{lin2017focal},
or sampling-based methods
\citep{menardi2014training,chawla2002smote}.
However, these AUC maximization methods require both labeled PN data, and thus, they cannot be applied to our problem setting.

PU learning methods train binary classifiers by using only PU data
\citep{bekker2020learning}.
A representative approach is empirical risk minimization, which rewrites the empirical risk by using only PU data
\citep{sugiyama2022machine}.
Although most methods are designed for ordinary classification on balanced data
\citep{du2015convex,kiryo2017positive,jiang2023positive},
several studies have recently shown the AUC can also be maximized from PU data
\citep{sakai2018semi,xie2018semi,charoenphakdee2019symmetric,xie2024weakly}.
However, these PU learning and AUC maximization methods assume the SCAR setting, i.e., labeled positive data are unbiased, which is not often satisfied in practice.
This paper tackles the more realistic but challenging SAR setting where labeled positive data are biased. 

Some PU learning methods for the SAR setting have recently been proposed
\citep{bekker2019beyond,kato2018learning,gerych2022recovering,he2018instance,teisseyre2025learning}.
Although several studies used EM algorithm-based methods
\citep{bekker2019beyond,gong2021instance}, 
they lack theoretical guarantees for learning the true classifier
\citep{gerych2022recovering,teisseyre2025learning}.
To make the problem theoretically tractable, 
most studies impose restrictive assumptions about data distributions or labeling mechanisms, such as
the data separation assumption
\citep{gerych2022recovering},
the invariance of order assumption
\citep{kato2018learning},
and the probabilistic gap assumption
\citep{he2018instance,gerych2022recovering}.
However, when the assumption is violated, they cannot work well.
In addition, they are not designed for maximizing the AUC.
In contrast, the proposed method can maximize the AUC from biased PU data using positive-confidence, without relying on such restrictive assumptions.

Several confidence-based learning methods have recently been proposed
\citep{ishida2018binary,cao2021learning,berthon2021confidence,feng2021pointwise,wang2023binary}.
Especially,
\citet{ishida2018binary}
and
\citet{shinoda2021binary}
use positive-confidence to learn binary classifiers.
Although they do not require unlabeled data, they cannot be applied for maximizing the AUC, and require that labeled positive data are unbiased. 
One PU learning method uses positive-confidence information to deal with noisy labels
\citep{tang2025confidence}.
However, it is not designed for either AUC maximization or the SAR setting.
In addition, existing confidence-based learning methods typically formulate confidence as true posterior probabilities.
In contrast, the proposed method can maximize the AUC even if the observed confidence is given by any strictly increasing transformation of the true posterior.

\section{Preliminaries}
\label{sec:prelim}


We briefly introduce AUC maximization.
Let input instance ${\bf x} \in \mathcal{X} $ and its class label $y \in \{0,1\}$ be associated with probability density $p({\bf x},y)$,
where $1$ and $0$ correspond to PN classes, respectively.
The conditional probability densities for PN classes are denoted as
$p^{{\rm p}} ({\bf x}) := p({\bf x} | y=1)$ and $p^{{\rm n}} ({\bf x}) := p({\bf x} | y=0)$.
Let $s: \mathcal{X} \to \mathbb{R} $ be a score function that quantifies how likely an instance is positive.
The classifier is defined by the score function with threshold $t$: $\hat{y}= I ( s({\bf x}) \geq t )$, where $I(z)$ is the indicator function that outputs $1$ if $z$ is true and $0$ otherwise.

The AUC is the probability that a randomly drawn positive instance is ranked higher than a randomly drawn negative instance with ties counted as one half
\citep{yang2022auc}.
Formally, the AUC with score function $s$ can be represented as
\begin{align}
{\rm AUC} (s) &= \mathbb{E}_{{\bf x}^{{\rm p}} \sim p^{{\rm p}} ({\bf x} )} \mathbb{E}_{ {\bf x}^{{\rm n}} \sim p^{{\rm n}} ({\bf x} )} \left[ {\ell}_{\rm{01}} (s({\bf x}^{{\rm n}}) - s({\bf x}^{{\rm p}})) \right] \nonumber\\
&= 1 - \mathbb{E}_{{\bf x}^{{\rm p}} \sim p^{{\rm p}} ({\bf x} )} \mathbb{E}_{ {\bf x}^{{\rm n}} \sim p^{{\rm n}} ({\bf x} )} \left[ \ell_{{\rm 01}} (s({\bf x}^{{\rm p}}) - s({\bf x}^{{\rm n}})) \right],
\end{align}
where $\ell_{01}(z)$ is the zero-one loss, with
$\ell_{01}(z)=1$ if $z<0$, $0$ if $z>0$, and $1/2$ if $z=0$,
and $\mathbb{E}$ denotes the expectation.
Maximizing the AUC is equivalent to minimizing the following AUC risk,
\begin{align}
\label{auc_risk}
{\cal R} (s) := \mathbb{E}_{{\bf x}^{{\rm p}} \sim p^{{\rm p}} ({\bf x} )} \mathbb{E}_{ {\bf x}^{{\rm n}} \sim p^{{\rm n}} ({\bf x} )} \left[ \ell_{{\rm 01}} (s({\bf x}^{{\rm p}}) - s({\bf x}^{{\rm n}})) \right].
\end{align}
Since the gradient of the zero-one loss is zero almost everywhere, 
the AUC risk cannot be directly minimized with gradient descent methods.
To address this, a common approach is to use the smoothed AUC risk by replacing $\ell_{01}(z)$ with the sigmoid surrogate $\sigma(-z)$, where $\sigma(z) = 1/(1+{\rm exp}(-z))$
\citep{charoenphakdee2019symmetric}:
\begin{align}
\label{sauc}
{\cal R}_{\sigma} (s) \!:=\! \mathbb{E}_{{\bf x}^{{\rm p}} \sim p^{{\rm p}} ({\bf x} )} \mathbb{E}_{{\bf x}^{{\rm n}} \sim p^{{\rm n}} ({\bf x} )} \left[ \sigma (- s({\bf x}^{{\rm p}}) \!+\! s({\bf x}^{{\rm n}})) \right].
\end{align}

Suppose we have $N^{{\rm p}}$ positive instances $\{ {\bf x}^{{\rm p}} _1, \dots, {\bf x}^{{\rm p}} _{N^{{\rm p}}} \} $ drawn from $p^{{\rm p}} ({\bf x})$ and $N^{{\rm n}}$ negative instances  $\{ {\bf x}^{{\rm n}} _1, \dots, {\bf x}^{{\rm n}} _{N^{{\rm n}}} \} $ drawn from $p^{{\rm n}} ({\bf x})$,  the empirical estimator of Eq. \eqref{sauc} is given as
\begin{align}
{\widehat {\cal R}_{\sigma} } (s) = \frac{1}{N^{{\rm p}}N^{{\rm n}}} \sum_{n=1}^{N^{{\rm p}}} \sum_{m=1}^{N^{{\rm n}}} \sigma (-s({\bf x} _n^{{\rm p}})+s({\bf x} _m^{{\rm n}})).
\end{align}
By minimizing this empirical AUC risk w.r.t. the parameters of $s$, we can obtain good score functions to maximize the AUC
\citep{yang2022auc}.

\section{Proposed Method}
\label{prop_method_all}

\subsection{Problem Formulation}
\label{sec:prop_form}

We assume that there exists probability density $p({\bf x}, y, o)$, where 
${\bf x} \in \mathcal{X}$ is an input instance, $y \in \{ 0, 1\}$ is its class, and $o \in \{0, 1\}$ represents whether its class $y$ is observed (labeled). Here, $o=1$ denotes that its class is labeled, and $o=0$ denotes that its class is unlabeled. In PU learning setting, only positive data are labeled. This PU assumption can be formally represented as
\begin{align}
p(o=1| {\bf x}, y=0) = 0,
\end{align}
which means that negative data are never labeled
\citep{elkan2008learning}.
Propensity score $e({\bf x})$ is defined as
\begin{align}
e({\bf x}) := p(o=1|{\bf x}, y=1),
\end{align}
which represents the probability that positive instance ${\bf x}$ is labeled.
Most PU learning methods assume the SCAR setting, i.e., propensity score $e({\bf x})$ does not depend on feature values ${\bf x}$,
$e({\bf x}) = p(o=1|y=1)$
\citep{sakai2018semi,kiryo2017positive,xie2018semi,xie2024weakly}.
Although this assumption simplifies the problem, it is seldom met in practice.
Thus, this paper assumes the SAR setting,
\begin{align}
e({\bf x}) = p(o=1|{\bf x}, y=1) \not\equiv p(o=1| y=1),
\end{align}
which means that $e({\bf x})$ depends on feature values ${\bf x}$. 
This is a more realistic but difficult setting.
Following previous studies on biased PU learning
\citep{bekker2019beyond,bekker2020learning,gerych2022recovering},
we assume the positivity condition that
$e({\bf x})>0$ for $p^{\rm p}$-almost every ${\bf x}$.
 
For collecting PU data, there are two possible scenarios: the {\it one-sample scenario} (also known as the {\it censoring scenario})
\citep{elkan2008learning} 
and the {\it two-sample scenario} (also known as the {\it case-control scenario})~\citep{ward2009presence}. 
This paper focuses on the one-sample scenario since it is more commonly used in PU learning studies~\citep{bekker2020learning}.
We can easily modify the proposed method for the two-sample scenario, which is described in Section~\ref{ours_ts}.
In the one-sample scenario, a set of unlabeled data is first sampled from true marginal density $p({\bf x}) = \pi p^{{\rm p}} ({\bf x}) + (1-\pi) p^{{\rm n}} ({\bf x})$ where $\pi := p(y=1)$ is the positive class-prior.
Then, if instance ${\bf x}$ is positive, it can be labeled with probability $e({\bf x})$, 
and ${\bf x}$ remains unlabeled with probability $1-e({\bf x})$; if instance {\bf x} is negative, it is never labeled and thus ${\bf x}$ remains unlabeled with probability $1$. In our problem setting, we additionally assume that each labeled positive instance ${\bf x}$ 
is equipped with confidence $r({\bf x}) := p(y=1|{\bf x})$.
Note that this equality does not have to strictly hold as shown later. 
As a result, we are given PU data $X^{{\rm p}} \cup X$ with positive-confidence $R^{{\rm p}}$:
\begin{align}
\label{given_data}
X^{{\rm p}} &:= \{ {\bf x} ^{{\rm p}}_n \}_{n=1}^{N^{{\rm p}}} \sim p^{{\rm l}} ({\bf x}) := p({\bf x} | o=1), \ \ X := \{ {\bf x} _m \}_{m=1}^N \sim p^{{\rm u}} ({\bf x}) := p({\bf x} | o=0), \nonumber\\
R^{{\rm p}} &:= \{ r^{{\rm p}}_n | \ r^{{\rm p}}_n = p(y=1 | {\bf x} ^{{\rm p}}_n) \}_{n=1}^{N^{{\rm p}}} ,
\end{align}
where $p^{{\rm l}}({\bf x})$ and $p^{{\rm u}}({\bf x})$ are labeled (positive) and unlabeled densities, respectively.
From the data generation process, true marginal density $p({\bf x})$ can be also represented as
$p({\bf x}) = \alpha p^{{\rm l}} ({\bf x}) + (1-\alpha)  p^{{\rm u}} ({\bf x})$, where $c := p(o=1|y=1)$ and $\alpha := p(o=1) = c \pi$\footnote{$\alpha = p(o=1) = p(o=1|y=1)p(y=1) + p(o=1|y=0)p(y=0) = p(o=1|y=1)p(y=1) = c \pi$, where we used the PU property (negative data are never labeled, i.e., $p(o=1|y=0)=0$).}.
Note that $p^{{\rm l}} ({\bf x}) \neq p^{{\rm p}} ({\bf x})$ in the SAR setting.
Our goal is to learn score function $s$ that can maximize the AUC under the SAR setting by using given PU data with positive-confidence 
$X^{{\rm p}} \cup X \cup R^{{\rm p}}$.

\subsection{AUC Risk Estimator with Biased PU Data and Positive-confidence}

We derive the AUC risk estimator in the SAR setting from biased PU data with positive-confidence.
The objective function to be minimized is the following smoothed AUC risk,
\begin{align}
\label{obj_auc}
{\cal R}_{\sigma} (s) = \mathbb{E}_{{\bf x}^{{\rm p}} \sim p^{{\rm p}} ({\bf x} )} \mathbb{E}_{{\bf x}^{{\rm n}} \sim p^{{\rm n}}  ({\bf x} )} \left[ f ({\bf x}^{{\rm p}}, {\bf x}^{{\rm n}}) \right],
\end{align}
where we set $f({\bf x}^{{\rm p}}, {\bf x}^{{\rm n}}) := \sigma (-s({\bf x}^{{\rm p}})+s({\bf x}^{{\rm n}}))$.
Although this AUC risk depends on $p^{{\rm p}} ({\bf x} )$ and $p^{{\rm n}} ({\bf x} )$, data from both the distributions are unavailable in our setting; it seems impossible to calculate.
However, we show that it is in fact possible below.

First, by the definition
$p({\bf x}) = \pi p^{{\rm p}}({\bf x}) + (1-\pi)p^{{\rm n}}({\bf x})$,
the negative density can be represented as
\begin{align}
\label{neg_den}
p^{{\rm n}} ({\bf x}) = \frac{1}{1-\pi} \left[ p({\bf x}) - \pi p^{{\rm p}} ({\bf x}) \right].
\end{align}
By substituting Eq. \eqref{neg_den} into Eq. \eqref{obj_auc}, we obtain
\begin{align}
\label{obj_auc_won}
{\cal R}_{\sigma} (s) &= \frac{1}{1-\pi}\mathbb{E}_{{\bf x}^{{\rm p}} \sim p^{{\rm p}} ({\bf x} )} \mathbb{E}_{{\bf x} \sim p ({\bf x} )} \left[ f ({\bf x}^{{\rm p}}, {\bf x} ) \right] - \frac{\pi}{1-\pi} \mathbb{E}_{{\bf x}^{{\rm p}} \sim p^{{\rm p}} ({\bf x} )} \mathbb{E}_{{\bf {\bar x}}^{{\rm p}} \sim p^{{\rm p}}  ({\bf x} )} \left[ f ({\bf x}^{{\rm p}}, {\bf {\bar x}}^{{\rm p}}) \right].
\end{align}
Here, $\mathbb{E}_{{\bf x}^{{\rm p}} \sim p^{{\rm p}} ({\bf x} )} \mathbb{E}_{{\bf {\bar x}}^{{\rm p}} \sim p^{{\rm p}}  ({\bf x} )} \left[ f ({\bf x}^{{\rm p}}, {\bf {\bar x}}^{{\rm p}}) \right]=1/2$
due to the symmetry of sigmoid function $\sigma$
\citep{xie2018semi,charoenphakdee2019symmetric,xie2024weakly}
(The proof is described in Section~\ref{const_lem}).
Thus, the AUC risk can be represented as
\begin{align}
\label{obj_auc_won_c}
{\cal R}_{\sigma} (s) = \frac{1}{1-\pi}\mathbb{E}_{{\bf x}^{{\rm p}} \sim p^{{\rm p}} ({\bf x} )} \mathbb{E}_{{\bf x} \sim p ({\bf x} )} \left[ f ({\bf x}^{{\rm p}}, {\bf x} ) \right] - \frac{\pi}{2(1-\pi)}.
\end{align}

Labeled density $p^{{\rm l}} ({\bf x})$ and positive density $p^{{\rm p}} ({\bf x})$ are related as follows,
\begin{align}
\label{rel_po}
p^{{\rm l}} ({\bf x}) = \frac{e({\bf x})}{c} p^{{\rm p}}({\bf x}),
\end{align}
which means that a positive instance is first sampled from $p^{{\rm p}} ({\bf x})$ and then is observed as labeled data with probability in accordance with propensity score $e({\bf x})$
\citep{bekker2020learning}.
The derivation of Eq. \eqref{rel_po} is described in Section \ref{lem1}.
Under the positivity condition described in Section \ref{sec:prop_form},
Eq.~\eqref{rel_po} can be inverted as
$p^{{\rm p}}({\bf x}) = c p^{\rm l} ({\bf x})/e({\bf x})$ for $p^{\rm p}$-almost every ${\bf x}$.

Additionally, by the PU assumption,
$p(o=1|{\bf x})=e({\bf x})p(y=1|{\bf x})$,
as shown in Section~\ref{lem2}.
Thus, for $p(y=1|{\bf x})>0$, we obtain
\begin{align}
\label{e_sy}
e({\bf x}) = \frac{p(o=1|{\bf x})}{p(y=1|{\bf x})}.
\end{align}
Since
$p^{{\rm p}}({\bf x})
= p(y=1|{\bf x})p({\bf x})/\pi$,
we have $p(y=1|{\bf x})>0$ for
$p^{{\rm p}}$-almost every ${\bf x}$.
Thus, Eq.~\eqref{e_sy} can be used together with the inverted form of
Eq.~\eqref{rel_po} in the $p^{{\rm p}}$-expectation in
Eq.~\eqref{obj_auc_won_c}.\footnote{
The denominator in the resulting expectation is also nonzero almost surely.
In fact, since
$p^{{\rm l}}({\bf x})
= p(o=1|{\bf x})p({\bf x})/p(o=1)$,
we have $p(o=1|{\bf x})>0$ for
$p^{{\rm l}}$-almost every ${\bf x}$.
}
We can therefore rewrite Eq.~\eqref{obj_auc_won_c} as
\begin{align}
\label{obj_auc_ours}
{\cal R}_{\sigma} (s) &= \frac{c}{1-\pi}\mathbb{E}_{{\bf x}^{{\rm p}} \sim p^{{\rm l}} ({\bf x})} \mathbb{E}_{{\bf x} \sim p ({\bf x} )} \left[ \frac{p(y=1|{\bf x}^{{\rm p}})}{p(o=1|{\bf x} ^{{\rm p}})} f ({\bf x}^{{\rm p}}, {\bf x} ) \right] -\frac{\pi}{2(1-\pi)}.
\end{align}
From this equation, we can see that weight $p(y=1|{\bf x}^{{\rm p}})/p(o=1|{\bf x}^{{\rm p}})$ becomes larger for labeled positive instance ${\bf x}^{{\rm p}}$ with higher $p(y=1|{\bf x}^{{\rm p}})$ (i.e., more likely to be positive) and lower $p(o=1|{\bf x}^{{\rm p}})$ (i.e., rarer as labeled data).
By substituting $p({\bf x}) = \alpha p^{{\rm l}} ({\bf x}) + (1-\alpha)  p^{{\rm u}} ({\bf x})$ into Eq. \eqref{obj_auc_ours},
we obtain
\begin{align}
\label{obj_auc_ours2}
{\cal R}_{\sigma} (s) &= \frac{c}{1-\pi} \left[ \alpha \mathbb{E}_{{\bf x}^{{\rm p}} \sim p^{{\rm l}} ({\bf x})} \mathbb{E}_{{\bar {\bf x}}^{{\rm p}} \sim p^{{\rm l}} ({\bf x} )} \left[ \frac{p(y=1|{\bf x}^{{\rm p}})}{p(o=1|{\bf x} ^{{\rm p}})} f ({\bf x}^{{\rm p}}, {\bar {\bf x}}^{\rm p} ) \right]  \right. \nonumber\\
&+ \left. (1-\alpha) \mathbb{E}_{{\bf x}^{{\rm p}} \sim p^{{\rm l}} ({\bf x})} \mathbb{E}_{{\bf x} \sim p^{{\rm u}} ({\bf x} )} \left[ \frac{p(y=1|{\bf x}^{{\rm p}})}{p(o=1|{\bf x} ^{{\rm p}})} f ({\bf x}^{{\rm p}}, {\bf x} ) \right] \right] -\frac{\pi}{2(1-\pi)}.
\end{align}
Here, we train probabilistic classifier $\hat{u}({\bf x})$ to estimate $u({\bf x}):=p(o=1|{\bf x})$ using given training data $X^{{\rm p}} \cup X$.
We also estimate labeled probability $\alpha=p(o=1)$ as $\hat{\alpha} := N^{{\rm p}}/(N^{{\rm p}} + N)$.
In addition, the confidence for labeled positive instances, $r_n^{{\rm p}} = p(y=1|{\bf x}_n^{{\rm p}})$, is given in our setting.
Thus, by replacing $u({\bf x}) = p(o=1|{\bf x})$ and $\alpha$ with their estimates and each expectation with the sample average in Eq.~\eqref{obj_auc_ours2}, we obtain the following estimator of the AUC risk:
\begin{align}
\label{obj_auc_ours2_emp}
{\hat {\cal R}}_{\sigma} &(s) \!=\! \frac{c}{1-\pi} \left[ \frac{\hat{\alpha}}{N^{{\rm p}}(N^{{\rm p}}-1)} \sum_{n \neq m}^{N^{{\rm p}},N^{{\rm p}}} \frac{r_n^{{\rm p}}}{{\hat u}_n^{{\rm p}}} f ({\bf x}_n^{{\rm p}}, {\bar {\bf x}}_m^{\rm p} ) \right. + \left. \frac{(1-\hat{\alpha})}{N^{{\rm p}}N} \sum_{n, m=1}^{N^{{\rm p}},N} \frac{r_n^{{\rm p}}}{\hat{u} _n^{{\rm p}}} f ({\bf x}_n^{{\rm p}}, {\bf x}_m )  \right] -\frac{\pi}{2(1-\pi)}.
\end{align}
where $\hat{u}_n^{{\rm p}} := \hat{u}({\bf x}_n^{{\rm p}})$, and we omit the case of $n=m$ in the first term to avoid biases following 
\citep{sakai2018semi}.
Since coefficient $c/(1-\pi)$ and constant $\pi/(2(1-\pi))$ do not affect the optimization of $s$, we can safely ignore them for training.
This property is preferable in practice since $\pi$ and $c$ are difficult or impossible to estimate
\citep{bekker2020learning}.
Note that although we use the sigmoid function in the AUC risk in Eq. \eqref{obj_auc}, 
we can derive the AUC risk estimator of the same form in Eq.~\eqref{obj_auc_ours2_emp}
whenever we use symmetric functions (i.e., function $\sigma$ satisfying $\sigma(z)+\sigma(-z)=k$ for any $z \in \mathbb{R}$ and $k$ is a constant).
This is because the second term in Eq.~\eqref{obj_auc_won} also becomes constant. 
The symmetric functions include many common loss functions such as sigmoid, ramp, and unhinged functions
\citep{charoenphakdee2019symmetric}.
Even when using non-symmetric functions,
we can derive a slightly modified version of the AUC risk estimator, which also enables us to maximize the AUC without requiring $\pi$ and $c$. The details are described in Section \ref{nonsymauc}.
We further analyze the sensitivity of the rewritten risk to labeling probability estimation errors 
in Section \ref{apen:u_error}
and the generalization behavior of its empirical estimator
in Section \ref{apen:generalization}.

Algorithm \ref{arg} shows the training procedure of the proposed method with stochastic gradient descent methods.
We first set mini-batch sizes $U$ and $P$ so as to maintain the original label ratio (Line 1). 
Then, we train probabilistic binary classifier $\hat{u}({\bf x})$ to estimate $p(o=1|{\bf x})$ with PU data $X^{{\rm p}} \cup X$ (Line 2).
Then, we randomly sample PU data with positive-confidence from given data $X^{{\rm p}} \cup X \cup R^{{\rm p}}$ (Lines 4--5).
We calculate the loss in Eq.~\eqref{obj_auc_ours2_emp} for learning score function $s$ with trained classifier $\hat{u}({\bf x})$ (Line 6).
We update parameters of score function $s$ by using the gradient of the loss (Line 7).
In the gradient calculation of the last step, parameters of trained classifier $\hat{u}({\bf x})$ are frozen.

\subsection{Robustness to Strictly Increasing Transformations of Positive-Confidence}
\label{robsut_pcon}

By replacing true posterior $p(y=1|{\bf x})$ in Eq.~\eqref{obj_auc_ours} with observed confidence $r({\bf x})$, we obtain the objective, which is exactly equivalent to the original AUC risk when 
$r({\bf x})=p(y=1|{\bf x})$.
However, we show below that this exact equality is not necessary to recover the Bayes-optimal ranking for the AUC.
\begin{proposition}
Suppose that the available confidence is $r'({\bf x}) = h\left(p(y=1|{\bf x})\right)$, where $h:[0,1]\to[0,1]$ is an arbitrary strictly increasing function. 
Then, the objective obtained by replacing $p(y=1|{\bf x})$ with $r'({\bf x})$ in Eq. \eqref{obj_auc_ours}, i.e., 
\begin{align}
\label{obj_auc_ours_strans}
{\cal R}_{\sigma}^{h} (s) &= \frac{c}{1-\pi}\mathbb{E}_{{\bf x}^{{\rm p}} \sim p^{{\rm l}} ({\bf x})} \mathbb{E}_{{\bf x} \sim p ({\bf x} )} \left[ \frac{r'({\bf x}^{{\rm p}})}{p(o=1|{\bf x} ^{{\rm p}})} f ({\bf x}^{{\rm p}}, {\bf x} ) \right] -\frac{\pi}{2(1-\pi)}
\end{align}
induces the same Bayes-optimal ranking for the AUC as the original AUC risk in Eq. \eqref{obj_auc}.
\end{proposition}
The proof is provided in Section \ref{apen:monotonic_transform}.
This result shows that the proposed method can recover the Bayes-optimal ranking for the AUC even when the available confidence is given as any strictly increasing transformation of the true posterior probability which includes systematic over- or under-confidence.
In this sense, the proposed method can treat a broader class of confidence than true posterior probabilities.
\begin{algorithm}[t]
\caption{Training procedure of the proposed method}
\label{arg}
\begin{algorithmic}[1]
\REQUIRE Biased PU data with positive-confidence $X^{{\rm p}} \cup X \cup R^{{\rm p}}$ and total mini-batch size $M$
\ENSURE Model parameters of score function $s$
\STATE{Partition $M$ into unlabeled and positive mini-batch sizes, $U$ and $P$, so that $P/(P+U) = N^{{\rm p}}/(N^{{\rm p}}+N)=: \hat{\alpha}$}
\STATE{Train classifier $\hat{u}({\bf x})$ to estimate $p(o=1|{\bf x})$ with PU data $X^{{\rm p}} \cup X$}
\REPEAT
\STATE{Sample unlabeled data with size $U$ from $X$}
\STATE{Sample positive data with size $P$ and their confidence from $X^{{\rm p}} \cup R^{{\rm p}}$}
\STATE{Calculate the loss in Eq. \eqref{obj_auc_ours2_emp} on the sampled data with trained classifier $\hat{u}({\bf x})$}
\STATE{Update parameters of $s$ with the gradient of the loss}
\UNTIL{End condition is satisfied;}
\end{algorithmic}
\end{algorithm}

\section{Experiments}
\label{exps}

\subsection{Data}
\label{sec:data}

We mainly used six real-world datasets: 
Mnist
\citep{lecun1998gradient}, 
FashionMnist (Fmnist)
\citep{xiao2017fashion}, 
Svhn
\citep{netzer2011reading}, 
Cifar10
\citep{krizhevsky2009learning},
Diabetes
\citep{gardner2024benchmarking},
and 
Blood
\citep{gardner2024benchmarking}.
The first four datasets are image datasets, while the remaining two are tabular datasets.
These datasets have been commonly used in PU learning studies
\citep{kumagaiauc,kumagaipositive,kumagai2025importance,xie2024weakly,kiryo2017positive,jiang2023positive}.
Following the previous studies
\citep{kumagai2025importance,kumagaiauc,kumagaipositive,xie2024weakly}, 
we constructed binary classification problems for the image datasets by partitioning their original classes into positive and negative classes, while using the original binary labels for Diabetes and Blood.
The detailed construction is provided in Section~\ref{apen:dataset}.

For training and validation in each dataset, we first randomly sampled $5,000$ and $1,000$ instances from the original data
with positive class-prior $\pi$, respectively.
To create the class-imbalanced data, we set $\pi$ to a small value.
Specifically, we changed $\pi$ within $\{ 0.05, 0.1, 0.15, 0.2 \}$ in our experiments.
Then, we set positive labeling rate $c=p(o=1|y=1)$ to $0.1$\footnote{Since the number of labeled positive training data $N^{{\rm p}}$ is determined as $N^{{\rm p}} = V p(o=1) = V p(o=1|y=1)p(y=1) = V c \pi$, where $V$ denotes the number of initially sampled training instances (i.e., $V=5,000)$,
$N^{{\rm p}}$ becomes $25$, $50$, $75$, and $100$ for the respective values of $\pi$.}.
We then selected a fraction $c$ of the positive instances within the initially sampled data as labeled positive data according to a biased sampling scheme.
Specifically, to induce sampling bias, we made instances from a subset of the positive classes more likely to be labeled for the image datasets and introduced a distance-based selection bias for the tabular datasets.
The detailed selection procedure for each dataset and results under other selection biases are provided in Sections \ref{apen:dataset} and \ref{sec_insdep_bias}, respectively.
We used $2,500$ positive and $2,500$ negative data as test data for evaluation.

In the main experiments, positive-confidence was estimated using a probabilistic classifier trained on $10,000$ labeled instances with class-prior $\pi$, following the standard evaluation protocol in previous confidence-based learning studies
\citep{ishida2018binary,cao2021learning,berthon2021confidence,wang2023binary}.
Note that these labeled data were used solely to generate confidence values for controlled evaluation.
This protocol enables us to quantitatively assess the effectiveness of our AUC risk estimator.
Confidence obtained from pre-trained classifiers is also practically relevant, as such classifiers are widely deployed in many real-world systems.
Note that results with positive-confidence obtained from real-world human annotators are also reported in Section \ref{sec:results}.
All labeled positive data for training and validation were equipped with positive-confidence.
Training data, validation data, test data, and data for the classifier to estimate positive-confidence did not overlap.
We conducted 10 experiments for each positive class-prior $\pi$ changing the random seeds and evaluated the mean test AUC.

\subsection{Comparison Methods}
\label{subsec:comp}

We compared the proposed method (Ours) with eight methods:
non-traditional classifier-based method (NTC)
\citep{elkan2008learning},
non-negative PU learning method (nnPU)
\citep{kiryo2017positive},
AUC maximization method with PU data (PUAUC)
\citep{xie2018semi,xie2024weakly},
PU learning method with selection bias (PUSB)
\citep{kato2018learning},
probabilistic gap-based method (PG)
\citep{gerych2022recovering},
positive-confidence-based method (Pconf)
\citep{ishida2018binary},
confidence-based noisy PU learning method (NPU)
\citep{tang2025confidence},
and the proposed method without positive-confidence (w/oConf).
We further evaluated LBE, a PU learning method under the SAR setting in Section \ref{apen:lbe}.
All methods including our method used neural networks for modeling classifiers.

NTC learns the binary classifier with PU data by simply treating all unlabeled data as negative data.
The proposed method also used it for estimating $p(o=1|{\bf x})$.
nnPU and PUAUC are PU learning methods for the SCAR setting (i.e., they assume that labeled positive data are unbiased).
Specifically, nnPU learns the classifier by minimizing the non-negative PU risk with PU data.
PUAUC learns the classifier by minimizing the AUC risk that is calculated with PU data.
PUSB and PG are PU learning methods for the SAR setting (i.e., they assume that labeled positive data are biased).
To deal with the SAR setting, PUSB assumes that $p(o=1|{\bf x})$ and $p(y=1|{\bf x})$ induce the same order on the feature space.
PG assumes propensity score $e({\bf x})$ to be a linear function of $p(y=1|{\bf x})$.
PUSB and PG do not use positive-confidence information.
In contrast, Pconf and NPU use positive-confidence information as in the proposed method.
Pconf learns the classifier by using only positive data with confidence.
NPU learns the classifier by using PU data with positive-confidence, which is designed to handle label noise in labeled positive data.
w/oConf is the proposed method that assumes $r_n^{{\rm p}} = 1$ for all labeled positive data in Eq. \eqref{obj_auc_ours2_emp}.
We compared it to evaluate the effectiveness of using positive-confidence in our framework. 
For nnPU, PUSB, and NPU, the absolute loss correction was used for preventing overfitting
\citep{kumagai2025importance}.

For the proposed method and w/oConf, we rounded up $\hat{u}_n^{{\rm p}} = \hat{u}({\bf x}_n^{{\rm p}})$ in Eq. \eqref{obj_auc_ours2_emp} 
less than $0.01$ to $0.01$ to stabilize the
training process following the previous study
\citep{ishida2018binary}.
We empirically confirm the robustness to this clipping value in Section \ref{apen:clipping}.
The empirical risk with validation PU data was used for early-stopping to mitigate overfitting.
All methods were implemented using PyTorch
\citep{paszke2017automatic}.
The details of settings such as network architectures are described in Section \ref{arc}.

\subsection{Results}
\label{sec:results}

Table \ref{result_all} shows the average test AUCs of each method on the six real-world datasets.
Full results including the standard deviations are described in Section \ref{sec:full_result}.
The proposed method performed the best or comparably to it in all cases.
NTC tended not to work well since it treats all unlabeled data as negative, which leads to biased results.
nnPU and PUAUC did not work well since they are not designed for the SAR setting.
In particular, since the main difference between the proposed method and PUAUC is whether the SAR setting is taken into account or not, 
this result indicates the importance of considering the bias of labeled positive data in AUC maximization framework.
PUSB and PG, which are PU learning methods for the SAR setting, tended to perform worse than the proposed method 
because their assumptions to deal with the SAR setting were invalid in our experiments.
By using positive-confidence information, the proposed method was able to perform well without relying on such assumptions.
Since Pconf does not use unlabeled data and assumes that labeled positive data are unbiased, it also did not work well even with positive-confidence. NPU, which uses positive-confidence in the noisy PU learning framework, also did not work well since it is not designed for either AUC maximization or the SAR setting.
The proposed method often outperformed w/oConf, which is the proposed method without positive-confidence.
Thus, this result indicates the effectiveness of using positive-confidence information in our framework.
Overall, these results show the effectiveness of the proposed method.
The results with each $\pi$ value are described in Section \ref{pi_graph}. The proposed method worked well in each case.

\begin{table*}[t]
\caption{Average test AUCs over different positive class-prior $\pi$ within $\{0.05, 0.1, 0.15, 0.2\}$. We set $c=p(o=1|y=1)=0.1$. 
Values in bold are not statistically different at the $5\%$ level from the best performing method in each row according to a paired t-test.
}
\label{result_all}
\centering
\scalebox{0.95}{
\begin{tabular}{lrrrrrrrrrr}
\hline
Data & \multicolumn{1}{c}{Ours} & \multicolumn{1}{c}{w/oConf} & \multicolumn{1}{c}{NTC} & \multicolumn{1}{c}{nnPU} & \multicolumn{1}{c}{PUAUC} &  \multicolumn{1}{c}{PUSB} & \multicolumn{1}{c}{PG} & \multicolumn{1}{c}{Pconf} & \multicolumn{1}{c}{NPU} \\
\hline
Mnist & \bf{0.8979} & 0.8769 & 0.8575 & 0.8196 & 0.8828 & 0.8841 & 0.8575 & 0.8193 & 0.8285  \\
Fmnist & \bf{0.9495} & 0.9371 & 0.9165 & 0.9296 & 0.9409 & \bf{0.9474} & 0.9165 & 0.7737 & \bf{0.9470}  \\
Svhn & \bf{0.6673} & \bf{0.6666} & 0.5328 & 0.4885 & \bf{0.6704} & 0.5965 & 0.5328 & 0.4900 & 0.4893  \\
Cifar10 & \bf{0.8643} & 0.8371 & 0.8013 & 0.4819 & \bf{0.8620} & 0.8442 & 0.8013 & 0.4767 & 0.4883  \\
Diabetes & \bf{0.7342} & 0.6926 & \bf{0.7357} & 0.5234 & 0.7171 & 0.7230 & \bf{0.7357} & 0.3974 & 0.4626  \\
Blood & \bf{0.5509} & 0.5287 & 0.5409 & 0.4847 & 0.5395 & \bf{0.5450} & 0.5409 & 0.4962 & 0.4828  \\
\hline
\end{tabular}
}
\end{table*}

We evaluated the performance of the proposed method when the confidence was distorted by strictly increasing transformation $h$.
As described in Section \ref{robsut_pcon}, minimizing the proposed AUC objective with such transformed confidence still leads to the Bayes-optimal ranking for the AUC.
Thus, the proposed method is expected to be empirically robust to such transformations.
We considered transformation $h(r)=r^k$, where $k>0$, with $k=1$ corresponding to the original confidence.
Figure~\ref{fig:strans_main_rk} shows the average test AUCs with their standard errors when changing $k$ of the transformation $h(r)=r^k$.
Overall, the performance of the proposed method remained relatively stable across different transformations and showed no systematic degradation as $k$ moved away from $1$.
In contrast, the performances of Pconf that also uses positive-confidence substantially degraded for some values of $k$ because it does not have the ranking-invariance property of the proposed method.
Results for other transformations are also reported in Section \ref{apend:strans_results}.
We also evaluated the robustness of the proposed method to additive Gaussian noise in positive-confidence in Section \ref{apend:results_gnoise}.

\begin{figure*}[t]
    \centering
    \begin{minipage}{0.23\linewidth}
        \centering
        \includegraphics[width=3.4cm]{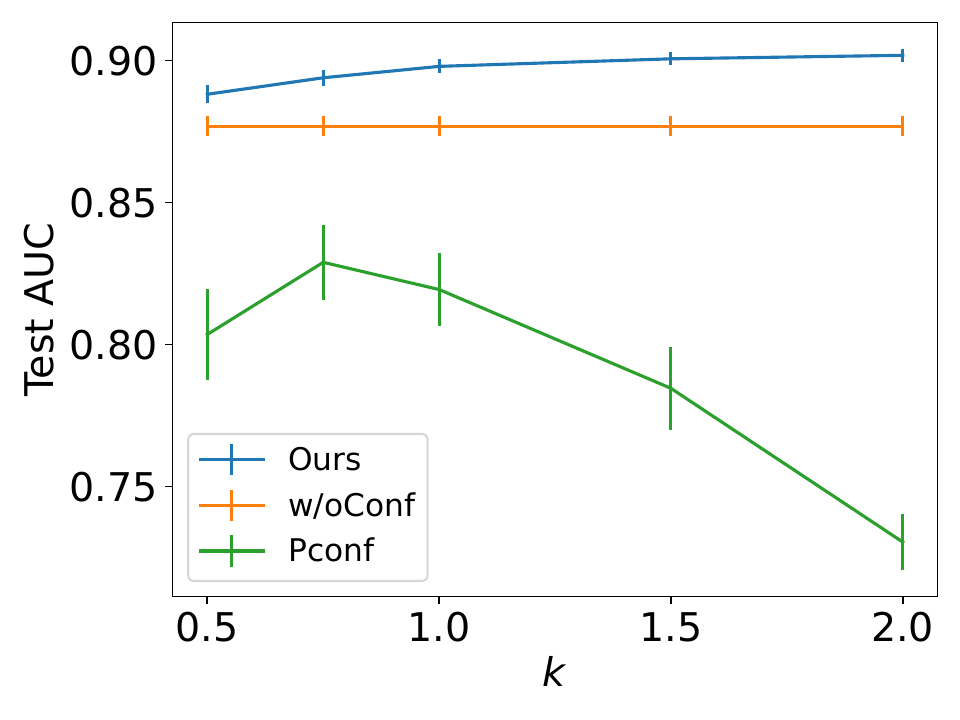}
        \subcaption{Mnist}
    \end{minipage}
    \hfill
    \begin{minipage}{0.23\linewidth}
        \centering
        \includegraphics[width=3.4cm]{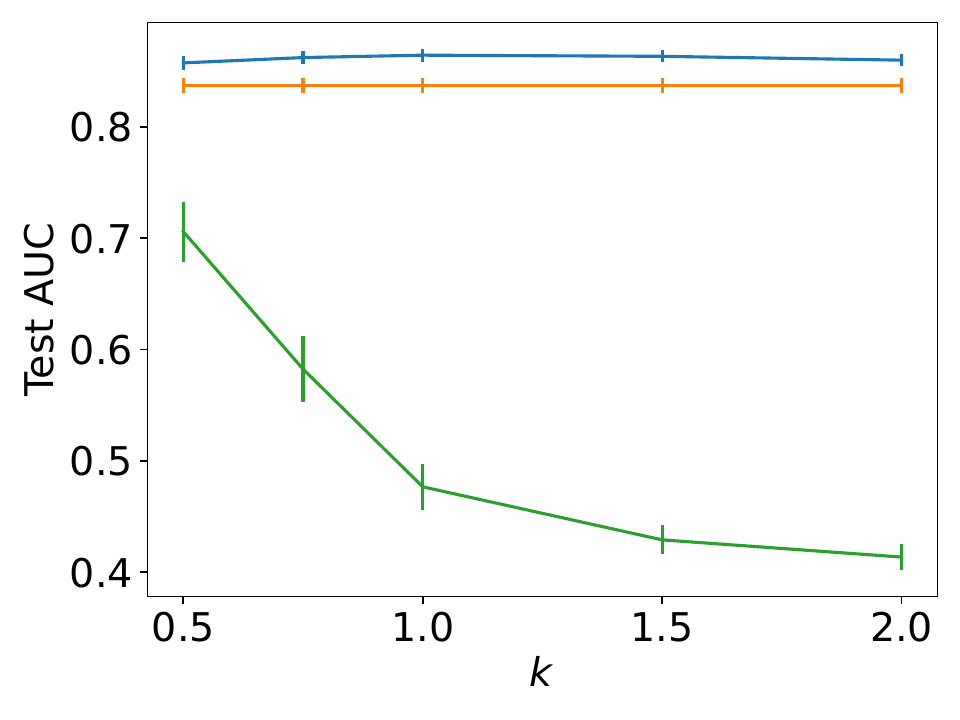}
        \subcaption{Cifar10}
    \end{minipage}
    \hfill
    \begin{minipage}{0.23\linewidth}
        \centering
        \includegraphics[width=3.4cm]{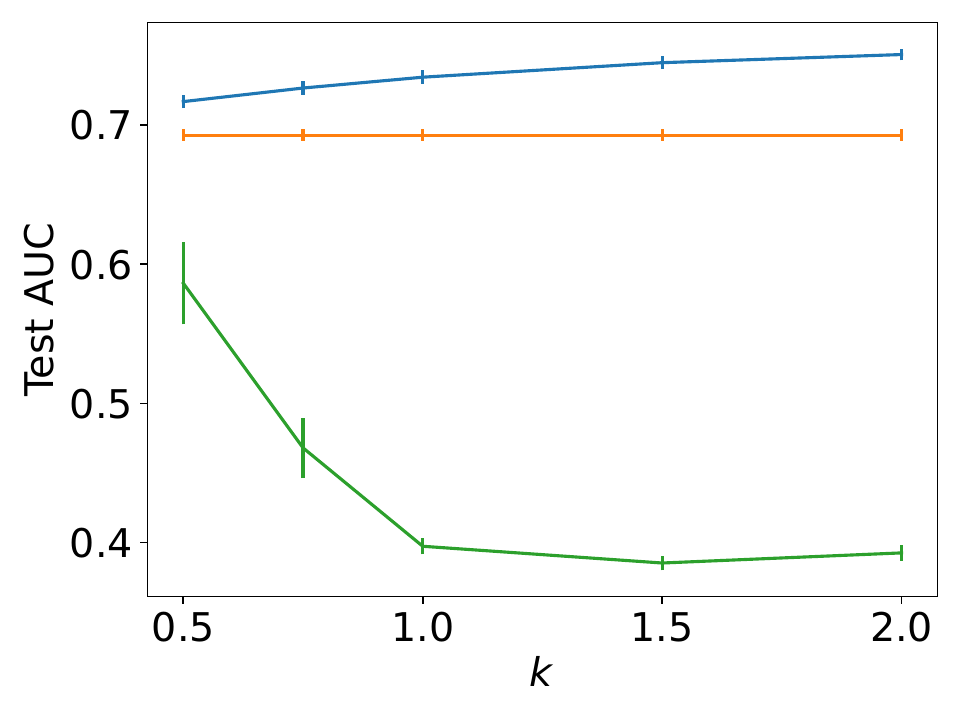}
        \subcaption{Diabetes}
    \end{minipage}
    \hfill
    \begin{minipage}{0.23\linewidth}
        \centering
        \includegraphics[width=3.4cm]{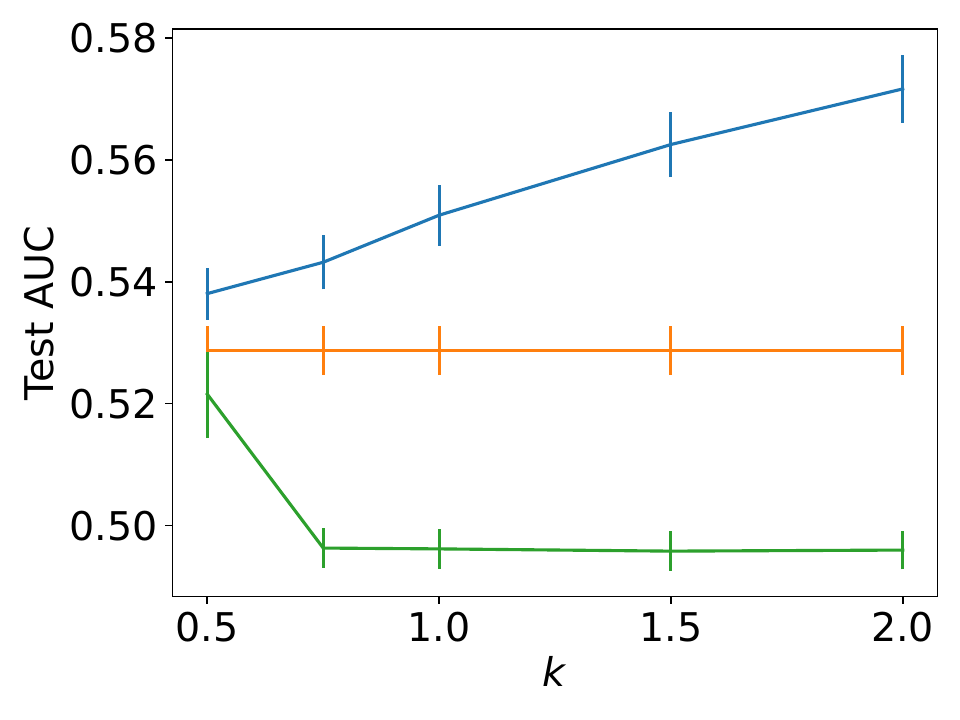}
        \subcaption{Blood}
    \end{minipage}
    \caption{The average test AUCs and their standard errors over different positive class-priors for different parameters $k$ of the transformation $h(r)=r^k$. Due to space limitation, we reported the results on two image and two tabular datasets; results on all datasets are provided in Section~\ref{apend:strans_results}.}
    \label{fig:strans_main_rk}
\end{figure*}
\begin{table*}[t!]
\caption{Results on the Cifar10-H and Fmnist-H datasets: average test AUCs over different positive class-prior $\pi$ within $\{0.05, 0.1, 0.15, 0.2\}$. We set $c=p(o=1|y=1)=0.1$. 
Values in bold are not statistically different at the $5\%$ level from the best performing method in each row according to a paired t-test.
}
\label{result_cifar10h}
\centering
\scalebox{0.9}{
\begin{tabular}{lrrrrrrrrrr}
\hline
Data & \multicolumn{1}{c}{Ours} & \multicolumn{1}{c}{w/oConf} & \multicolumn{1}{c}{NTC} & \multicolumn{1}{c}{nnPU} & \multicolumn{1}{c}{PUAUC} &  \multicolumn{1}{c}{PUSB} & \multicolumn{1}{c}{PG} & \multicolumn{1}{c}{Pconf} & \multicolumn{1}{c}{NPU} \\
\hline
Cifar10-H & \bf{0.5991} & \bf{0.5998} & 0.4779 & 0.4404 & 0.5813 & 0.5106 & 0.4779 & 0.5828 & 0.5712  \\
Fmnist-H & \bf{0.8285} & 0.8192 & 0.7823 & 0.6013 & \bf{0.8324} & 0.8242 & 0.7823 & 0.5211 & 0.8061  \\
\hline
\end{tabular}
}
\end{table*}

In our experiments so far, we have used classifier-based confidence. 
However, confidence can also be obtained from labels provided by multiple annotators. 
Therefore, we evaluated this setting using the Cifar10-H and Fmnist-H datasets, where each instance was labeled by multiple real-world human annotators
\citep{peterson2019human,ishidaperformance},
and which have been widely used in confidence-based learning studies~\citep{ishidaperformance,ushio2025practical}.
Confidence for each instance can be obtained by averaging the corresponding annotations
\citep{ishidaperformance}.
Details of the datasets and data construction are provided in Section~\ref{apen:human_data}.
Table \ref{result_cifar10h} shows the results.
On both datasets, the proposed method achieved performance statistically comparable
to the best-performing method, demonstrating its applicability to human-derived confidence.
On Cifar10-H, the average positive-confidence value was close to one.
Therefore, the proposed method behaved similarly to w/oConf, which
sets all positive-confidence values to one, resulting in similar performance.

\section{Conclusion}

In this paper, we proposed an AUC maximization method under the SAR setting with biased PU data and positive-confidence.
By using positive-confidence, the proposed method can maximize the AUC even in the SAR setting without restrictive assumptions about the data distributions or labeling mechanism.
Our experiments on eight real-world datasets demonstrate the effectiveness of the proposed method.

\bibliography{ref_aucda}
\bibliographystyle{iclr2027_conference}

\appendix

\section{Extended Related Work}

Kumagai et al.
\citeyearpar{kumagaipositive}
have recently proposed a distribution shift adaptation method for AUC maximization.
Specifically, as distribution shift, this paper assumes the covariate shift, which refers to the situation in which the distribution of input instance changes while the input-output (class label) relationship remains unchanged between the training and test phases.  
This method uses not only PU data from the training phase but also unlabeled data from the test phase to address the covariate shift.
The PU data in the training phase can be regarded as biased data.
Specifically, 
although the distributions of labeled positive data and unlabeled positive data in the training phase are identical 
(i.e., the SCAR setting), the distributions of labeled positive data in the training phase and the unlabeled positive data in the test phase may differ due to the covariate shift (i.e., a SAR-like setting). While this approach requires a large amount of unlabeled data obtained in the training phase to handle covariate shift, such unlabeled data are not available in our problem setting. Instead, the proposed method requires confidence information for a small amount of positive data.

\section{Theoretical Analysis}
\label{sec:theory}

We first show a basic relationship between the supports of the labeling probability and the posterior probability, and derive an expectation identity used in the following analysis.

\begin{lemma}
\label{lem:support_relation}
Let
$\eta({\bf x}):=p(y=1|{\bf x})$,
$u({\bf x}):=p(o=1|{\bf x})$, and
$\alpha:=p(o=1)$.
Under the PU assumption and the positivity condition in
Section~\ref{sec:prop_form},
\begin{align}
I(u({\bf x})>0)
=
I(\eta({\bf x})>0)
\quad
p\text{-almost surely},
\label{eq:support_u_eta}
\end{align}
and $u({\bf x})>0$ for $p^{\rm l}$-almost every ${\bf x}$.
Moreover, for any measurable function $\varphi$ such that the following expectations are finite,
\begin{align}
\mathbb E_{{\bf x}\sim p^{\rm l}({\bf x})}
\left[
\frac{\varphi({\bf x})}{u({\bf x})}
\right]
=
\frac{1}{\alpha}
\mathbb E_{{\bf x}\sim p({\bf x})}
\left[
\varphi({\bf x})
I(\eta({\bf x})>0)
\right].
\label{eq:change_measure_pl_p}
\end{align}
\end{lemma}

\begin{proof}
By Eq.~\eqref{eq:labeling_propensity} in
Section~\ref{lem2},
\begin{align}
u({\bf x})=e({\bf x})\eta({\bf x}).
\end{align}
Therefore, $u({\bf x})>0$ implies
$\eta({\bf x})>0$.

To show the reverse implication, the positivity condition implies
\begin{align}
\int_{\{{\bf x}:e({\bf x})=0\}}
p^{\rm p}({\bf x})d{\bf x}
=0.
\end{align}
Therefore, letting
$A:=\{{\bf x}:e({\bf x})=0,\ \eta({\bf x})>0\}$,
we have
\begin{align}
0
&=
\int_A p^{\rm p}({\bf x})d{\bf x}
\nonumber\\
&=
\frac{1}{\pi}
\int_A
\eta({\bf x})p({\bf x})d{\bf x},
\end{align}
where we used
$p^{\rm p}({\bf x})=\eta({\bf x})p({\bf x})/\pi$.
Since $\eta({\bf x})>0$ on $A$, this implies
\begin{align}
\int_A p({\bf x})d{\bf x}=0.
\end{align}
Thus,
$e({\bf x})>0$ for $p$-almost every ${\bf x}$ such that
$\eta({\bf x})>0$.
Since $u({\bf x})=e({\bf x})\eta({\bf x})$, it follows that
\begin{align}
\eta({\bf x})>0
\quad\Longrightarrow\quad
u({\bf x})>0
\qquad p\text{-almost surely}.
\end{align}

Combining the two implications, we obtain
\begin{align}
I(u({\bf x})>0)
=
I(\eta({\bf x})>0)
\quad p\text{-almost surely}.
\end{align}

Next, by Bayes' rule,
\begin{align}
p^{\rm l}({\bf x})
=
\frac{u({\bf x})p({\bf x})}{\alpha}.
\label{eq:pl_up_common}
\end{align}
Therefore,
\begin{align}
\int_{\{{\bf x}:u({\bf x})=0\}}
p^{\rm l}({\bf x})d{\bf x}
&=
\frac{1}{\alpha}
\int_{\{{\bf x}:u({\bf x})=0\}}
u({\bf x})p({\bf x})d{\bf x}
\nonumber\\
&=0.
\end{align}
Thus, $u({\bf x})>0$ for $p^{\rm l}$-almost every
${\bf x}$.
We can therefore restrict the expectation under $p^{\rm l}$
to the region where $u({\bf x})>0$. Using
Eq.~\eqref{eq:pl_up_common}, we obtain
\begin{align}
&
\mathbb E_{{\bf x}\sim p^{\rm l}({\bf x})}
\left[
\frac{\varphi({\bf x})}{u({\bf x})}
\right]
\nonumber\\
&=
\int_{\{{\bf x}:u({\bf x})>0\}}
p^{\rm l}({\bf x})
\frac{\varphi({\bf x})}{u({\bf x})}
d{\bf x}
\nonumber\\
&=
\frac{1}{\alpha}
\int_{\{{\bf x}:u({\bf x})>0\}}
p({\bf x})\varphi({\bf x})
d{\bf x}
\nonumber\\
&=
\frac{1}{\alpha}
\mathbb E_{{\bf x}\sim p({\bf x})}
\left[
\varphi({\bf x})
I(u({\bf x})>0)
\right]
\nonumber\\
&=
\frac{1}{\alpha}
\mathbb E_{{\bf x}\sim p({\bf x})}
\left[
\varphi({\bf x})
I(\eta({\bf x})>0)
\right],
\end{align}
where the last equality follows from
Eq.~\eqref{eq:support_u_eta}.
\end{proof}

\subsection{Proof of The Robustness Statement under Strictly Increasing Transformations of Positive-confidence}
\label{apen:monotonic_transform}

\begin{proof}
Let
$\eta({\bf x}):=p(y=1|{\bf x})$ and
$u({\bf x}):=p(o=1|{\bf x})$.
Since the multiplicative coefficient in
Eq.~\eqref{obj_auc_ours_strans} is positive and the additive
term is independent of $s$, minimizing
${\cal R}_{\sigma}^{h}(s)$ is equivalent to minimizing
\begin{align}
\widetilde{\cal R}_{\sigma}^{h}(s)
:=
\mathbb E_{{\bf x}^{\rm p}\sim p^{\rm l}({\bf x})}
\mathbb E_{{\bf x}\sim p({\bf x})}
\left[
\frac{h(\eta({\bf x}^{\rm p}))}
{u({\bf x}^{\rm p})}
f({\bf x}^{\rm p},{\bf x})
\right].
\end{align}

Applying Lemma~\ref{lem:support_relation} with
\begin{align}
\varphi({\bf x}^{\rm p})
=
h(\eta({\bf x}^{\rm p}))
\mathbb E_{{\bf x}\sim p({\bf x})}
[f({\bf x}^{\rm p},{\bf x})],
\end{align}
we obtain
\begin{align}
\widetilde{\cal R}_{\sigma}^{h}(s)
=
\frac{1}{\alpha}
\mathbb E_{{\bf x}^{\rm p},{\bf x}\sim p({\bf x})}
\left[
h(\eta({\bf x}^{\rm p}))
I(\eta({\bf x}^{\rm p})>0)
f({\bf x}^{\rm p},{\bf x})
\right].
\label{eq:transformed_population_risk}
\end{align}

Define
\begin{align}
g_h(z):=h(z)I(z>0),
\end{align}
and
\begin{align}
Z_h
&:=
\mathbb E_{{\bf x}\sim p({\bf x})}
[g_h(\eta({\bf x}))],
&
q_h({\bf x})
&:=
\frac{g_h(\eta({\bf x}))p({\bf x})}{Z_h}.
\end{align}
Since
$\pi=\mathbb E_{{\bf x}\sim p({\bf x})}[\eta({\bf x})]>0$
and $g_h(z)>0$ for $z>0$, we have $Z_h>0$.
Then,
\begin{align}
\widetilde{\cal R}_{\sigma}^{h}(s)
=
\frac{Z_h}{\alpha}
\mathbb E_{{\bf x}^{\rm p}\sim q_h({\bf x})}
\mathbb E_{{\bf x}\sim p({\bf x})}
[f({\bf x}^{\rm p},{\bf x})].
\end{align}

Thus, up to a positive multiplicative constant,
$\widetilde{\cal R}_{\sigma}^{h}$ is the sigmoid AUC surrogate
for bipartite ranking between $q_h$ and $p$.
By the AUC consistency of the sigmoid loss
\citep{charoenphakdee2019symmetric}
and the Bayes characterization of bipartite ranking
\citep{menon2016bipartite},
the Bayes-optimal ranking for this problem is determined by
\begin{align}
\frac{q_h({\bf x})}{p({\bf x})}
=
\frac{g_h(\eta({\bf x}))}{Z_h}.
\end{align}

Since $h:[0,1]\to[0,1]$ is strictly increasing,
$g_h$ is also strictly increasing on $[0,1]$:
for $z>0$,
$g_h(z)=h(z)>h(0)\geq0=g_h(0)$,
and for $0<z_1<z_2$,
$g_h(z_1)=h(z_1)<h(z_2)=g_h(z_2)$.
Thus, the Bayes-optimal ranking induced by
$q_h({\bf x})/p({\bf x})$ is identical to that induced by
$\eta({\bf x})$.

On the other hand, the Bayes-optimal ranking for the original
AUC risk in Eq.~\eqref{obj_auc} is determined by
\begin{align}
\frac{p^{\rm p}({\bf x})}{p^{\rm n}({\bf x})}
=
\frac{1-\pi}{\pi}
\frac{\eta({\bf x})}{1-\eta({\bf x})},
\end{align}
which is also strictly increasing in $\eta({\bf x})$.
Thus, the transformed objective and the original AUC risk
induce the same Bayes-optimal ranking for the AUC.
\end{proof}

\subsection{Sensitivity to Labeling-Probability Estimation}
\label{apen:u_error}

In practice, the labeling probability $p(o=1|{\bf x})$ in
Eq.~\eqref{obj_auc_ours} is unknown and must be estimated from the
given PU data.
Here, we analyze how its estimation error affects the rewritten AUC
risk and the resulting ranking.

\begin{proposition}
Let $\eta({\bf x}):=p(y=1|{\bf x})$ and
$u({\bf x}):=p(o=1|{\bf x})$.
For an estimate $\widehat u({\bf x})$ of $u({\bf x})$, define
\begin{align}
{\cal R}_{\sigma}(s;\widehat u)
:=
\frac{c}{1-\pi}
\mathbb E_{{\bf x}^{\rm p}\sim p^{\rm l}({\bf x})}
\mathbb E_{{\bf x}\sim p({\bf x})}
\left[
\frac{\eta({\bf x}^{\rm p})}
{\widehat u({\bf x}^{\rm p})}
f({\bf x}^{\rm p},{\bf x})
\right]
-\frac{\pi}{2(1-\pi)}.
\end{align}
Note that
${\cal R}_{\sigma}(s;u)={\cal R}_{\sigma}(s)$
in Eq.~\eqref{obj_auc_ours}.
Suppose that
$\widehat u({\bf x})\geq\delta>0$ for all ${\bf x}$ and
\begin{align}
|\widehat u({\bf x})-u({\bf x})|
\leq\epsilon
\end{align}
for $p$-almost every ${\bf x}$ satisfying $u({\bf x})>0$.
Then,
\begin{align}
\sup_s
\left|
{\cal R}_{\sigma}(s;\widehat u)
-
{\cal R}_{\sigma}(s;u)
\right|
\leq
\frac{\epsilon}{(1-\pi)\delta}.
\label{eq:u_error_bound}
\end{align}

Furthermore, suppose that $\widehat s$ is an approximate
minimizer satisfying
\begin{align}
{\cal R}_{\sigma}(\widehat s;\widehat u)
\leq
\inf_s{\cal R}_{\sigma}(s;\widehat u)
+\xi_{\rm opt},
\end{align}
where $\xi_{\rm opt}\ge0$.
Then,
\begin{align}
{\cal R}_{\sigma}(\widehat s;u)
-
\inf_s{\cal R}_{\sigma}(s;u)
\leq
\frac{2\epsilon}{(1-\pi)\delta}
+\xi_{\rm opt}.
\end{align}
\end{proposition}

\begin{proof}
Since $0\leq f\leq1$,
\begin{align}
&
\left|
{\cal R}_{\sigma}(s;\widehat u)
-
{\cal R}_{\sigma}(s;u)
\right|
\nonumber\\
&\leq
\frac{c}{1-\pi}
\mathbb E_{{\bf x}\sim p^{\rm l}({\bf x})}
\left[
\eta({\bf x})
\frac{|\widehat u({\bf x})-u({\bf x})|}
{u({\bf x})\widehat u({\bf x})}
\right].
\end{align}
Applying Lemma~\ref{lem:support_relation} with
\begin{align}
\varphi({\bf x})
=
\eta({\bf x})
\frac{|\widehat u({\bf x})-u({\bf x})|}
{\widehat u({\bf x})},
\end{align}
we obtain
\begin{align}
&
\left|
{\cal R}_{\sigma}(s;\widehat u)
-
{\cal R}_{\sigma}(s;u)
\right|
\nonumber\\
&\leq
\frac{c}{\alpha(1-\pi)}
\mathbb E_{{\bf x}\sim p({\bf x})}
\left[
\eta({\bf x})
\frac{|\widehat u({\bf x})-u({\bf x})|}
{\widehat u({\bf x})}
I(\eta({\bf x})>0)
\right]
\nonumber\\
&\leq
\frac{c\epsilon}{\alpha(1-\pi)\delta}
\mathbb E_{{\bf x}\sim p({\bf x})}
\left[
\eta({\bf x})I(\eta({\bf x})>0)
\right]
\nonumber\\
&=
\frac{c\epsilon}{\alpha(1-\pi)\delta}
\mathbb E_{{\bf x}\sim p({\bf x})}
[\eta({\bf x})]
=
\frac{\epsilon}{(1-\pi)\delta},
\end{align}
where the second inequality follows from
$I(\eta({\bf x})>0)=I(u({\bf x})>0)$
$p$-almost surely in
Lemma~\ref{lem:support_relation} and the assumptions on
$\widehat u$.
In the last equality, we used
$\mathbb E_{p}[\eta({\bf x})]=\pi$ and
$\alpha=c\pi$.
By taking the supremum over $s$, we obtain
Eq.~\eqref{eq:u_error_bound}.

By Eq.~\eqref{eq:u_error_bound},
\begin{align}
{\cal R}_{\sigma}(\widehat s;u)
&\leq
{\cal R}_{\sigma}(\widehat s;\widehat u)
+
\frac{\epsilon}{(1-\pi)\delta}
\nonumber\\
&\leq
\inf_s{\cal R}_{\sigma}(s;\widehat u)
+\xi_{\rm opt}
+
\frac{\epsilon}{(1-\pi)\delta}
\nonumber\\
&\leq
\inf_s{\cal R}_{\sigma}(s;u)
+\xi_{\rm opt}
+
\frac{2\epsilon}{(1-\pi)\delta},
\end{align}
which proves the second claim.
\end{proof}

Eq.~\eqref{eq:u_error_bound} shows that, for fixed $\pi$ and
$\delta$, the deviation of the rewritten AUC risk increases at most
linearly with the labeling probability estimation error $\epsilon$.
The factor $1/\delta$ also indicates that the effect of estimation
errors can be amplified when the estimated labeling probability is
close to zero.
This provides theoretical motivation for the clipping used in our experiments.

We next show how misspecification of the labeling
probability affects the resulting ranking.
Using
$p^{\rm l}({\bf x})=u({\bf x})p({\bf x})/\alpha$,
the part of ${\cal R}_{\sigma}(s;\widehat u)$ that depends on
$s$ can, up to a positive multiplicative constant, be written as
\begin{align}
\mathbb E_{{\bf x}^{\rm p},{\bf x}\sim p({\bf x})}
\left[
\eta({\bf x}^{\rm p})
\frac{u({\bf x}^{\rm p})}
{\widehat u({\bf x}^{\rm p})}
f({\bf x}^{\rm p},{\bf x})
\right].
\end{align}
Define
\begin{align}
\widetilde\eta({\bf x})
:=
\eta({\bf x})
\frac{u({\bf x})}{\widehat u({\bf x})}.
\end{align}
By the same Bayes characterization of bipartite ranking used in
Section~\ref{apen:monotonic_transform}
\citep{menon2016bipartite},
the Bayes-optimal ranking associated with this objective is
determined by $\widetilde\eta({\bf x})$.

This result shows that some estimation errors in the
labeling probability do not affect the resulting ranking.
First, suppose that
$\widehat u({\bf x})=a u({\bf x})$ for some constant $a>0$
for $p$-almost every ${\bf x}$ satisfying $u({\bf x})>0$.
On this set,
\begin{align}
\widetilde\eta({\bf x})
=
\eta({\bf x})
\frac{u({\bf x})}{\widehat u({\bf x})}
=
\frac{\eta({\bf x})}{a}.
\end{align}
On the other hand, by Lemma~\ref{lem:support_relation},
$u({\bf x})=0$ implies $\eta({\bf x})=0$
$p$-almost surely.
In addition, since $\widetilde\eta({\bf x})=0$ when $u({\bf x})=0$, we have
\begin{align}
\widetilde\eta({\bf x})
=
\frac{\eta({\bf x})}{a}
\quad
p\text{-almost surely}.
\end{align}

This almost-sure equality is sufficient for the population AUC. Since
$p({\bf x})=\pi p^{\rm p}({\bf x})+(1-\pi)p^{\rm n}({\bf x})$
with $0<\pi<1$, any $p$-null set is also null under both
$p^{\rm p}$ and $p^{\rm n}$.
Hence, the pairwise ordering induced by
$\widetilde\eta$ and $\eta/a$ is identical
$p^{\rm p}\times p^{\rm n}$-almost surely.
Since $a>0$, $\eta({\bf x})/a$ preserves the ordering induced by
$\eta({\bf x})$.
Thus, a constant multiplicative error in the estimated labeling
probability does not change the Bayes-optimal AUC ranking.

More generally, suppose that the relative estimation error is
uniformly bounded as
\begin{align}
\left|
\frac{\widehat u({\bf x})}{u({\bf x})}-1
\right|
\leq \rho < 1
\label{eq:relative_u_error}
\end{align}
for $p$-almost every ${\bf x}$ satisfying $u({\bf x})>0$.
Then,
\begin{align}
1-\rho
\leq
\frac{\widehat u({\bf x})}{u({\bf x})}
\leq
1+\rho,
\end{align}
and thus
\begin{align}
\frac{\eta({\bf x})}{1+\rho}
\leq
\widetilde\eta({\bf x})
\leq
\frac{\eta({\bf x})}{1-\rho}.
\end{align}
For two instances ${\bf x}$ and $\bar{\bf x}$ outside the
corresponding $p$-null set satisfying
$\eta({\bf x})>\eta(\bar{\bf x})>0$,
their ordering is guaranteed to be preserved if the smallest
possible value of $\widetilde\eta({\bf x})$ is larger than the
largest possible value of $\widetilde\eta(\bar{\bf x})$, namely, if
\begin{align}
\frac{\eta({\bf x})}{1+\rho}
>
\frac{\eta(\bar{\bf x})}{1-\rho},
\end{align}
or equivalently,
\begin{align}
\frac{\eta({\bf x})}{\eta(\bar{\bf x})}
>
\frac{1+\rho}{1-\rho}.
\end{align}
If $\eta(\bar{\bf x})=0<\eta({\bf x})$, their ordering is also
preserved outside the corresponding $p$-null set.
Therefore, under the relative-error bound in Eq.~\eqref{eq:relative_u_error}, a pair
can have its ordering reversed only when their posterior
probabilities are sufficiently close.

\subsection{Generalization Error Analysis}
\label{apen:generalization}

We next provide a finite-sample analysis of learning with the proposed
AUC risk estimator.
We consider finite score function and labeling probability classes to analyze the basic convergence behavior of the estimator.

Let $\mathcal S$ be a finite class of score functions, and let
$\mathcal U_\delta$ be a finite class of labeling probability functions
such that
\begin{align}
v({\bf x})\ge\delta>0
\quad
\text{for all }
v\in\mathcal U_\delta
\text{ and }{\bf x}.
\end{align}
Define
$K:=|\mathcal S||\mathcal U_\delta|$.
Let $\widetilde r({\bf x})\in[0,1]$ denote the observed, possibly
noisy confidence.
To make the dependence on the score function explicit, we write
\begin{align}
f_s({\bf x},{\bf x}')
:=
\sigma(-s({\bf x})+s({\bf x}')).
\end{align}
For $s\in\mathcal S$ and $v\in\mathcal U_\delta$, define
\begin{align}
h_{s,v}({\bf x},{\bf x}')
:=
\frac{\widetilde r({\bf x})}{v({\bf x})}
f_s({\bf x},{\bf x}').
\label{eq:h_sv}
\end{align}
Since
$0\le\widetilde r({\bf x})\le1$,
$0\le f_s({\bf x},{\bf x}')\le1$, and
$v({\bf x})\ge\delta$, we have
\begin{align}
0\le h_{s,v}({\bf x},{\bf x}')\le\frac{1}{\delta}.
\label{eq:h_sv_bound}
\end{align}

Define the empirical positive-positive and positive-unlabeled terms
corresponding to Eq.~\eqref{obj_auc_ours2_emp} as
\begin{align}
\widehat L_{\rm pp}(s,v)
&:=
\frac{1}{N^{\rm p}(N^{\rm p}-1)}
\sum_{n\ne m}^{N^{\rm p},N^{\rm p}}
h_{s,v}({\bf x}^{\rm p}_n,{\bf x}^{\rm p}_m),
\label{eq:Lpp_emp}\\
\widehat L_{\rm pu}(s,v)
&:=
\frac{1}{N^{\rm p}N}
\sum_{n=1}^{N^{\rm p}}
\sum_{m=1}^{N}
h_{s,v}({\bf x}^{\rm p}_n,{\bf x}_m).
\label{eq:Lpu_emp}
\end{align}
The corresponding population terms are
\begin{align}
L_{\rm pp}(s,v)
&:=
\mathbb{E}_{{\bf x}\sim p^{\rm l}({\bf x})}
\mathbb{E}_{{\bf x}'\sim p^{\rm l}({\bf x})}
[h_{s,v}({\bf x},{\bf x}')],
\label{eq:Lpp_pop}\\
L_{\rm pu}(s,v)
&:=
\mathbb{E}_{{\bf x}\sim p^{\rm l}({\bf x})}
\mathbb{E}_{{\bf x}'\sim p^{\rm u}({\bf x})}
[h_{s,v}({\bf x},{\bf x}')].
\label{eq:Lpu_pop}
\end{align}
Let $\alpha:=p(o=1)$ and define
\begin{align}
L(s,v)
:=
\alpha L_{\rm pp}(s,v)
+(1-\alpha)L_{\rm pu}(s,v).
\label{eq:L_pop}
\end{align}

We first derive a uniform finite-sample bound.
For $\tau\in(0,1)$, define
\begin{align}
B_{N^{\rm p},N}(\tau)
:=
\frac{1}{\delta}
\left[
\alpha
\sqrt{
\frac{2\log(4K/\tau)}{N^{\rm p}}
}
+
(1-\alpha)
\sqrt{
\frac{\log(4K/\tau)}{2}
\left(
\frac{1}{N^{\rm p}}+\frac{1}{N}
\right)
}
\right].
\label{eq:B_generalization}
\end{align}

\begin{lemma}
With probability at least $1-\tau$,
\begin{align}
\sup_{\substack{s\in\mathcal S\\v\in\mathcal U_\delta}}
\left|
\alpha\widehat L_{\rm pp}(s,v)
+
(1-\alpha)\widehat L_{\rm pu}(s,v)
-
L(s,v)
\right|
\le
B_{N^{\rm p},N}(\tau).
\label{eq:uniform_L_bound}
\end{align}
\end{lemma}

\begin{proof}
Consider a fixed pair $(s,v)$.
Replacing one labeled positive instance changes at most
$2(N^{\rm p}-1)$ summands in
$\widehat L_{\rm pp}(s,v)$.
Thus, by Eq.~\eqref{eq:h_sv_bound}, replacing one labeled positive instance
changes $\widehat L_{\rm pp}(s,v)$ by at most
$2/(N^{\rm p}\delta)$.
As a result, McDiarmid's inequality gives
\begin{align}
\Pr\left(
\left|
\widehat L_{\rm pp}(s,v)-L_{\rm pp}(s,v)
\right|
\ge t
\right)
\le
2\exp\left(
-\frac{N^{\rm p}\delta^2t^2}{2}
\right).
\label{eq:mcdiarmid_pp}
\end{align}

For $\widehat L_{\rm pu}(s,v)$, replacing one labeled positive
instance changes $N$ summands and thus changes the average by at most
$1/(N^{\rm p}\delta)$.
Similarly, replacing one unlabeled instance changes $N^{\rm p}$
summands and thus changes the average by at most $1/(N\delta)$.
As a result, McDiarmid's inequality gives
\begin{align}
\Pr\left(
\left|
\widehat L_{\rm pu}(s,v)-L_{\rm pu}(s,v)
\right|
\ge t
\right)
\le
2\exp\left(
-\frac{2\delta^2t^2}
{1/N^{\rm p}+1/N}
\right).
\label{eq:mcdiarmid_pu}
\end{align}

Applying a union bound over the
$K=|\mathcal S||\mathcal U_\delta|$ pairs in
$\mathcal S\times\mathcal U_\delta$ and over the two empirical terms
yields, with probability at least $1-\tau$,
\begin{align}
\sup_{\substack{s\in\mathcal S\\v\in\mathcal U_\delta}}
\left|
\widehat L_{\rm pp}(s,v)-L_{\rm pp}(s,v)
\right|
&\le
\frac{1}{\delta}
\sqrt{
\frac{2\log(4K/\tau)}{N^{\rm p}}
},
\\
\sup_{\substack{s\in\mathcal S\\v\in\mathcal U_\delta}}
\left|
\widehat L_{\rm pu}(s,v)-L_{\rm pu}(s,v)
\right|
&\le
\frac{1}{\delta}
\sqrt{
\frac{\log(4K/\tau)}{2}
\left(
\frac{1}{N^{\rm p}}+\frac{1}{N}
\right)
}.
\end{align}
Eq.~\eqref{eq:uniform_L_bound} follows by combining these two
inequalities with Eq.~\eqref{eq:L_pop}.
\end{proof}

In practice, $\alpha$ is estimated as
\begin{align}
\widehat\alpha
:=
\frac{N^{\rm p}}{N^{\rm p}+N}.
\end{align}
Define its estimation error as
\begin{align}
\epsilon_\alpha
:=
|\widehat\alpha-\alpha|.
\label{eq:alpha_error_def}
\end{align}
Since
$0\le\widehat L_{\rm pp}(s,v),\widehat L_{\rm pu}(s,v)\le1/\delta$,
we have
\begin{align}
&
\left|
\widehat\alpha\widehat L_{\rm pp}(s,v)
+
(1-\widehat\alpha)\widehat L_{\rm pu}(s,v)
-
\alpha\widehat L_{\rm pp}(s,v)
-
(1-\alpha)\widehat L_{\rm pu}(s,v)
\right|
\nonumber\\
&=
|\widehat\alpha-\alpha|
\left|
\widehat L_{\rm pp}(s,v)-\widehat L_{\rm pu}(s,v)
\right|
\nonumber\\
&\le
\frac{\epsilon_\alpha}{\delta}.
\end{align}
Thus, on the event in Eq.~\eqref{eq:uniform_L_bound},
\begin{align}
\sup_{\substack{s\in\mathcal S\\v\in\mathcal U_\delta}}
&
\left|
\widehat\alpha\widehat L_{\rm pp}(s,v)
+
(1-\widehat\alpha)\widehat L_{\rm pu}(s,v)
-
L(s,v)
\right|
\nonumber\\
&\le
B_{N^{\rm p},N}(\tau)
+
\frac{\epsilon_\alpha}{\delta}.
\label{eq:uniform_L_alpha}
\end{align}
Thus, Eq.~\eqref{eq:uniform_L_alpha} also holds with
probability at least $1-\tau$.

For arbitrary confidence $\widetilde r$ and
$v\in\mathcal U_\delta$, define the population and empirical
rewritten risks by
\begin{align}
{\cal R}_{\sigma}(s;\widetilde r,v)
&:=
\frac{c}{1-\pi}L(s,v)
-\frac{\pi}{2(1-\pi)},
\label{eq:R_general_pop}\\
\widehat{\cal R}_{\sigma}(s;\widetilde r,v)
&:=
\frac{c}{1-\pi}
\left[
\widehat\alpha\widehat L_{\rm pp}(s,v)
+
(1-\widehat\alpha)\widehat L_{\rm pu}(s,v)
\right]
-\frac{\pi}{2(1-\pi)}.
\label{eq:R_general_emp}
\end{align}
Define
\begin{align}
\Gamma_{N^{\rm p},N}(\tau)
:=
\frac{c}{1-\pi}
\left[
B_{N^{\rm p},N}(\tau)
+
\frac{\epsilon_\alpha}{\delta}
\right].
\label{eq:Gamma_generalization}
\end{align}
Then, Eq.~\eqref{eq:uniform_L_alpha} implies that, with probability
at least $1-\tau$,
\begin{align}
\sup_{\substack{s\in\mathcal S\\v\in\mathcal U_\delta}}
\left|
\widehat{\cal R}_{\sigma}(s;\widetilde r,v)
-
{\cal R}_{\sigma}(s;\widetilde r,v)
\right|
\le
\Gamma_{N^{\rm p},N}(\tau).
\label{eq:uniform_R_bound}
\end{align}

We next define the population risk corresponding to the true
confidence and labeling probability.
For the following analysis, we interpret
$\eta({\bf x})/u({\bf x})$ as zero when $u({\bf x})=0$.
By Lemma~\ref{lem:support_relation},
$u({\bf x})>0$ for $p^{\rm l}$-almost every ${\bf x}$.
Therefore, this convention does not affect expectations with
respect to $p^{\rm l}$ and thus does not change the rewritten
AUC risk in Eq.~\eqref{obj_auc_ours}.

We denote this population risk by
\begin{align}
{\cal R}_{\sigma}(s;\eta,u)
:=
\frac{c}{1-\pi}
\mathbb E_{{\bf x}^{\rm p}\sim p^{\rm l}({\bf x})}
\mathbb E_{{\bf x}\sim p({\bf x})}
\left[
\frac{\eta({\bf x}^{\rm p})}
{u({\bf x}^{\rm p})}
f_s({\bf x}^{\rm p},{\bf x})
\right]
-
\frac{\pi}{2(1-\pi)}.
\label{eq:R_true_generalization}
\end{align}
By Eq.~\eqref{obj_auc_ours},
${\cal R}_{\sigma}(s;\eta,u)$ is equal to the original population
AUC risk.

We next show the excess-risk bound for the score function learned
with the proposed estimator.

\begin{theorem}
Suppose that the data-dependent estimate
$\widehat u\in\mathcal U_\delta$ and that
$\widehat s\in\mathcal S$ is an approximate empirical minimizer
satisfying
\begin{align}
\widehat{\cal R}_{\sigma}
(\widehat s;\widetilde r,\widehat u)
\le
\inf_{s\in\mathcal S}
\widehat{\cal R}_{\sigma}
(s;\widetilde r,\widehat u)
+
\xi_{\rm opt},
\label{eq:approx_erm}
\end{align}
where $\xi_{\rm opt}\ge0$.
Suppose that the PU assumption and the positivity condition in
Section~\ref{sec:prop_form} hold, and that
\begin{align}
|\widetilde r({\bf x})-\eta({\bf x})|
&\le\epsilon_r,
\nonumber\\
|\widehat u({\bf x})-u({\bf x})|
&\le\epsilon_u
\label{eq:nuisance_errors}
\end{align}
for $p$-almost every ${\bf x}$ satisfying $u({\bf x})>0$.
Then, with probability at least $1-\tau$,
\begin{align}
{\cal R}_{\sigma}(\widehat s;\eta,u)
-
\inf_{s\in\mathcal S}
{\cal R}_{\sigma}(s;\eta,u)
\le
2\Gamma_{N^{\rm p},N}(\tau)
+
\frac{2(c\epsilon_r+\epsilon_u)}
{(1-\pi)\delta}
+
\xi_{\rm opt}.
\label{eq:final_generalization}
\end{align}
\end{theorem}

\begin{proof}
On the event in Eq.~\eqref{eq:uniform_R_bound}, the uniform bound
holds simultaneously for all
$(s,v)\in\mathcal S\times\mathcal U_\delta$.
Therefore, since $\widehat u\in\mathcal U_\delta$, it also holds
for the data-dependent estimate $\widehat u$:
\begin{align}
\sup_{s\in\mathcal S}
\left|
\widehat{\cal R}_{\sigma}(s;\widetilde r,\widehat u)
-
{\cal R}_{\sigma}(s;\widetilde r,\widehat u)
\right|
\le
\Gamma_{N^{\rm p},N}(\tau).
\label{eq:uniform_R_hatu}
\end{align}

Since $\mathcal S$ is finite, let
\begin{align}
s^\dagger
\in
\argmin_{s\in\mathcal S}
{\cal R}_{\sigma}(s;\widetilde r,\widehat u).
\end{align}
Then, by Eqs.~\eqref{eq:uniform_R_hatu} and \eqref{eq:approx_erm}, 
we have
\begin{align}
{\cal R}_{\sigma}
(\widehat s;\widetilde r,\widehat u)
&\le
\widehat{\cal R}_{\sigma}
(\widehat s;\widetilde r,\widehat u)
+
\Gamma_{N^{\rm p},N}(\tau)
\nonumber\\
&\le
\widehat{\cal R}_{\sigma}
(s^\dagger;\widetilde r,\widehat u)
+
\Gamma_{N^{\rm p},N}(\tau)
+
\xi_{\rm opt}
\nonumber\\
&\le
{\cal R}_{\sigma}
(s^\dagger;\widetilde r,\widehat u)
+
2\Gamma_{N^{\rm p},N}(\tau)
+
\xi_{\rm opt}.
\end{align}
By the definition of $s^\dagger$, this yields
\begin{align}
{\cal R}_{\sigma}
(\widehat s;\widetilde r,\widehat u)
-
\inf_{s\in\mathcal S}
{\cal R}_{\sigma}
(s;\widetilde r,\widehat u)
\le
2\Gamma_{N^{\rm p},N}(\tau)
+
\xi_{\rm opt}.
\label{eq:erm_bound_noisy}
\end{align}

We next analyze the effects of the confidence error and the labeling-probability estimation error.
For any $s\in\mathcal S$, using $0\le f_s\le1$, we have
\begin{align}
&
\left|
{\cal R}_{\sigma}(s;\widetilde r,\widehat u)
-
{\cal R}_{\sigma}(s;\eta,u)
\right|
\nonumber\\
&\le
\frac{c}{1-\pi}
\mathbb E_{{\bf x}\sim p^{\rm l}({\bf x})}
\left[
\left|
\frac{\widetilde r({\bf x})}
{\widehat u({\bf x})}
-
\frac{\eta({\bf x})}
{u({\bf x})}
\right|
\right].
\label{eq:nuisance_bound_step1}
\end{align}
Here and below, $\eta({\bf x})/u({\bf x})$ is interpreted as zero
when $u({\bf x})=0$.
By Lemma~\ref{lem:support_relation},
$u({\bf x})>0$ for $p^{\rm l}$-almost every ${\bf x}$.
Thus, we can restrict the expectation in
Eq.~\eqref{eq:nuisance_bound_step1} to the set
$\{{\bf x}:u({\bf x})>0\}$.
On this set,
\begin{align}
\left|
\frac{\widetilde r({\bf x})}{\widehat u({\bf x})}
-
\frac{\eta({\bf x})}{u({\bf x})}
\right|
&\le
\frac{
|\widetilde r({\bf x})-\eta({\bf x})|
}
{\widehat u({\bf x})}
+
\eta({\bf x})
\frac{
|\widehat u({\bf x})-u({\bf x})|
}
{u({\bf x})\widehat u({\bf x})}.
\label{eq:nuisance_decomp}
\end{align}

Since $\widehat u({\bf x})\ge\delta$ and
$|\widetilde r({\bf x})-\eta({\bf x})|\le\epsilon_r$
for $p$-almost every ${\bf x}$ satisfying $u({\bf x})>0$,
the expectation of the first term is bounded as
\begin{align}
\mathbb E_{{\bf x}\sim p^{\rm l}({\bf x})}
\left[
\frac{
|\widetilde r({\bf x})-\eta({\bf x})|
}
{\widehat u({\bf x})}
\right]
\le
\frac{\epsilon_r}{\delta}.
\label{eq:r_error_component}
\end{align}

For the second term, using
$p^{\rm l}({\bf x})=u({\bf x})p({\bf x})/\alpha$, we obtain
\begin{align}
&
\mathbb E_{{\bf x}\sim p^{\rm l}({\bf x})}
\left[
I(u({\bf x})>0)
\eta({\bf x})
\frac{
|\widehat u({\bf x})-u({\bf x})|
}
{u({\bf x})\widehat u({\bf x})}
\right]
\nonumber\\
&=
\frac{1}{\alpha}
\mathbb E_{{\bf x}\sim p({\bf x})}
\left[
I(u({\bf x})>0)
\eta({\bf x})
\frac{
|\widehat u({\bf x})-u({\bf x})|
}
{\widehat u({\bf x})}
\right]
\nonumber\\
&\le
\frac{\epsilon_u}{\alpha\delta}
\mathbb E_{{\bf x}\sim p({\bf x})}
\left[
\eta({\bf x})I(u({\bf x})>0)
\right].
\label{eq:u_error_component_pre}
\end{align}
By Lemma~\ref{lem:support_relation},
$I(u({\bf x})>0)=I(\eta({\bf x})>0)$
$p$-almost surely.
Therefore,
\begin{align}
\mathbb E_{{\bf x}\sim p({\bf x})}
\left[
\eta({\bf x})I(u({\bf x})>0)
\right]
&=
\mathbb E_{{\bf x}\sim p({\bf x})}
[\eta({\bf x})]
\nonumber\\
&=
\pi.
\end{align}
Since $\alpha=c\pi$, Eq.~\eqref{eq:u_error_component_pre}
gives
\begin{align}
&
\mathbb E_{{\bf x}\sim p^{\rm l}({\bf x})}
\left[
I(u({\bf x})>0)
\eta({\bf x})
\frac{
|\widehat u({\bf x})-u({\bf x})|
}
{u({\bf x})\widehat u({\bf x})}
\right]
\le
\frac{\epsilon_u}{c\delta}.
\label{eq:u_error_component}
\end{align}
Combining Eqs.~\eqref{eq:nuisance_bound_step1},
\eqref{eq:r_error_component}, and
\eqref{eq:u_error_component}, we obtain
\begin{align}
\sup_{s\in\mathcal S}
\left|
{\cal R}_{\sigma}(s;\widetilde r,\widehat u)
-
{\cal R}_{\sigma}(s;\eta,u)
\right|
&\le
\frac{c}{1-\pi}
\left(
\frac{\epsilon_r}{\delta}
+
\frac{\epsilon_u}{c\delta}
\right)
\nonumber\\
&=
\frac{c\epsilon_r+\epsilon_u}
{(1-\pi)\delta}.
\label{eq:nuisance_bound}
\end{align}

Let
\begin{align}
s^\ast
\in
\argmin_{s\in\mathcal S}
{\cal R}_{\sigma}(s;\eta,u).
\end{align}
Using Eq.~\eqref{eq:nuisance_bound} for both
$\widehat s$ and $s^\ast$ and Eq.~\eqref{eq:erm_bound_noisy}, we obtain
\begin{align}
{\cal R}_{\sigma}(\widehat s;\eta,u)
&\le
{\cal R}_{\sigma}
(\widehat s;\widetilde r,\widehat u)
+
\frac{c\epsilon_r+\epsilon_u}
{(1-\pi)\delta}
\nonumber\\
&\le
{\cal R}_{\sigma}
(s^\ast;\widetilde r,\widehat u)
+
2\Gamma_{N^{\rm p},N}(\tau)
+
\xi_{\rm opt}
+
\frac{c\epsilon_r+\epsilon_u}
{(1-\pi)\delta}
\nonumber\\
&\le
{\cal R}_{\sigma}(s^\ast;\eta,u)
+
2\Gamma_{N^{\rm p},N}(\tau)
+
\xi_{\rm opt}
+
\frac{2(c\epsilon_r+\epsilon_u)}
{(1-\pi)\delta},
\end{align}
which proves Eq.~\eqref{eq:final_generalization}.
\end{proof}

The bound in Eq.~\eqref{eq:final_generalization} decomposes the
excess smoothed AUC risk into the finite-sample statistical error,
the estimation error of $\alpha$, the confidence error,
the labeling-probability estimation error, and the optimization error.
For fixed finite function classes,
$B_{N^{\rm p},N}(\tau)$ converges to zero as
$N^{\rm p},N\to\infty$.
Thus, the excess risk converges to zero as the sample sizes increase,
provided that the estimation errors of $\alpha$, the confidence,
and the labeling probability, as well as the optimization error,
also vanish.
Here, we consider finite function classes for a self-contained analysis.
Extending the result to infinite function classes using Rademacher complexity is left for future work.

\section{Derivations for the Proposed AUC Risk}

\subsection{Constancy of the Positive-positive Term}
\label{const_lem}

\begin{lemma}
$\mathbb{E}_{{\bf x}^{{\rm p}} \sim p^{{\rm p}} ({\bf x} )} \mathbb{E}_{{\bar {\bf x }^{{\rm p}} } \sim p^{{\rm p}} ({\bf x} )} \left[  f({\bf x}^{{\rm p}}, {\bar {\bf x}^{{\rm p}}}) \right] = \frac{1}{2}$.
\end{lemma}
\begin{proof}
Since $f({\bf x}^{{\rm p}}, {\bar {\bf x }^{{\rm p}}}) = \sigma (-s({\bf x}^{{\rm p}})+s({\bar {\bf x }^{{\rm p}}}))$ and $\sigma(z)+\sigma(-z)=1$ for all $z \in \mathbb{R}$,
\begin{align}
\mathbb{E}_{{\bf x}^{{\rm p}} \sim p^{{\rm p}} ({\bf x} )} \mathbb{E}_{{\bar {\bf x }^{{\rm p}} } \sim p^{{\rm p}} ({\bf x} )} \left[  f({\bf x}^{{\rm p}}, {\bar {\bf x}^{{\rm p}}}) \right] 
 =& \mathbb{E}_{{\bf x}^{{\rm p}} \sim p^{{\rm p}} ({\bf x} )} \mathbb{E}_{{\bar {\bf x }^{{\rm p}} } \sim p^{{\rm p}} ({\bf x} )} \left[ 1 - f({\bar {\bf x}^{{\rm p}}}, {\bf x}^{{\rm p}}) \right] \nonumber\\
=& 1 - \mathbb{E}_{{\bf x}^{{\rm p}} \sim p^{{\rm p}} ({\bf x} )} \mathbb{E}_{{\bar {\bf x }^{{\rm p}} } \sim p^{{\rm p}} ({\bf x} )} \left[ f({\bar {\bf x}^{{\rm p}}}, {\bf x}^{{\rm p}}) \right].
\end{align}
It is clear that the lemma follows from this equation.
\end{proof}

\subsection{Derivation of the Relation between the Labeled and Positive Densities}
\label{lem1}

\begin{lemma}
Labeled density $p^{{\rm l}}({\bf x}) = p({\bf x}|o=1)$ can be represented as $p^{{\rm l}}({\bf x}) = \frac{e({\bf x})}{c} p^{{\rm p}}({\bf x})$, where $e({\bf x}) = p(o=1|{\bf x}, y=1)$, $c=p(o=1|y=1)$, and $p^{{\rm p}}({\bf x})=p({\bf x}|y=1)$.
\end{lemma}
\begin{proof}
\begin{align}
p^{{\rm l}}({\bf x}) &= p({\bf x}|o=1) \nonumber\\
&=p({\bf x}|o=1, y=1) \nonumber\\
&=\frac{p({\bf x}, o=1, y=1)}{p(o=1, y=1)} \nonumber\\
&=\frac{p(o=1|{\bf x}, y=1)p({\bf x}|y=1)p(y=1)}{p(o=1|y=1)p(y=1)} \nonumber\\
&= \frac{e({\bf x})}{c} p^{{\rm p}}({\bf x}),
\end{align}
where we used the PU property (labeled data are always positive) in the second equality and the Bayes' rule in the third equality.
\end{proof}

\subsection{Relation between the Propensity Score and the Labeling Probability}
\label{lem2}

\begin{lemma}
Labeling probability $p(o=1|{\bf x})$ and propensity score
$e({\bf x})=p(o=1|{\bf x},y=1)$ satisfy
\begin{align}
p(o=1|{\bf x})
=
e({\bf x})p(y=1|{\bf x}).
\label{eq:labeling_propensity}
\end{align}
Consequently, for any ${\bf x}$ such that $p(y=1|{\bf x})>0$,
\begin{align}
e({\bf x})
=
\frac{p(o=1|{\bf x})}{p(y=1|{\bf x})}.
\label{eq:propensity_labeling}
\end{align}
\end{lemma}

\begin{proof}
By the law of total probability,
\begin{align}
p(o=1|{\bf x})
&=
p(o=1|{\bf x},y=1)p(y=1|{\bf x})
+
p(o=1|{\bf x},y=0)p(y=0|{\bf x}) \nonumber\\
&=
e({\bf x})p(y=1|{\bf x}),
\end{align}
where the second equality follows from the definition of the propensity score
and the PU assumption $p(o=1|{\bf x},y=0)=0$.
Dividing both sides by
$p(y=1|{\bf x})$ gives Eq.~\eqref{eq:propensity_labeling}.
Moreover, since
$p^{\rm p}({\bf x})=p(y=1|{\bf x})p({\bf x})/\pi$,
we have $p(y=1|{\bf x})>0$ for
$p^{\rm p}$-almost every ${\bf x}$.
Therefore, Eq.~\eqref{e_sy} holds for
$p^{\rm p}$-almost every ${\bf x}$.
\end{proof}

\section{Extensions of the Proposed Method}

\subsection{Extension to Non-symmetric Loss Functions}
\label{nonsymauc}

In this section, we derive the AUC risk estimator when using non-symmetric function $\sigma$, which also enables us to maximize the AUC without requiring $\pi$ and $c$.

Specifically, since the second term in Eq. \eqref{obj_auc_won} does not become a constant with non-symmetric functions, 
we need to calculate this term with biased positive data. 
Using Eqs.~\eqref{rel_po} and \eqref{e_sy}, as in the derivation
of the proposed risk in the main paper, the second term in
Eq.~\eqref{obj_auc_won} can be rewritten as
\begin{align}
\frac{\pi}{1-\pi}
&\mathbb{E}_{{\bf x}^{\rm p}\sim p^{\rm p}({\bf x})}
\mathbb{E}_{{\bar{\bf x}}^{\rm p}\sim p^{\rm p}({\bf x})}
\left[
f({\bf x}^{\rm p},{\bar{\bf x}}^{\rm p})
\right]
\nonumber\\
&=
\frac{c\alpha}{1-\pi}
\mathbb{E}_{{\bf x}^{\rm p}\sim p^{\rm l}({\bf x})}
\mathbb{E}_{{\bar{\bf x}}^{\rm p}\sim p^{\rm l}({\bf x})}
\left[
\frac{p(y=1|{\bf x}^{\rm p})}
     {p(o=1|{\bf x}^{\rm p})}
\frac{p(y=1|{\bar{\bf x}}^{\rm p})}
     {p(o=1|{\bar{\bf x}}^{\rm p})}
f({\bf x}^{\rm p},{\bar{\bf x}}^{\rm p})
\right],
\end{align}
where we used $\alpha=c\pi$.

By replacing the constant term $\frac{\pi}{2(1-\pi)}$ in Eq. \eqref{obj_auc_ours2} 
with the above-derived term, we can obtain the AUC risk with non-symmetric functions as follows:
\begin{align}
{\cal R} (s) &= \frac{c}{1-\pi} \left[ \alpha \mathbb{E}_{{\bf x}^{{\rm p}} \sim p^{{\rm l}} ({\bf x})} \mathbb{E}_{{\bar {\bf x}}^{{\rm p}} \sim p^{{\rm l}} ({\bf x} )} \left[ \frac{p(y=1|{\bf x}^{{\rm p}})}{p(o=1|{\bf x} ^{{\rm p}})} f ({\bf x}^{{\rm p}}, {\bar {\bf x}}^{\rm p} ) \right] \right. \nonumber\\ 
& \left. +(1-\alpha) \mathbb{E}_{{\bf x}^{{\rm p}} \sim p^{{\rm l}} ({\bf x})} \mathbb{E}_{{\bf x} \sim p^{{\rm u}} ({\bf x} )} \left[ \frac{p(y=1|{\bf x}^{{\rm p}})}{p(o=1|{\bf x} ^{{\rm p}})} f ({\bf x}^{{\rm p}}, {\bf x} ) \right] \right. \nonumber\\
& \left. - \alpha \mathbb{E}_{{\bf x}^{{\rm p}} \sim p^{{\rm l}} ({\bf x} )} \mathbb{E}_{{\bf {\bar x}}^{{\rm p}} \sim p^{{\rm l}} ({\bf x} )} \left[ \frac{p(y=1|{\bf x}^{{\rm p}})}{p(o=1|{\bf x}^{{\rm p}})} \frac{p(y=1|{\bf {\bar x}}^{{\rm p}})}{p(o=1|{\bf {\bar x}}^{{\rm p}})} f ({\bf x}^{{\rm p}}, {\bf {\bar x}}^{{\rm p}}) \right] \right].
\end{align}
The corresponding AUC risk estimator can be obtained by approximating the expectation with the sample average.
Since the coefficient $c/(1-\pi)$ does not affect the optimization of $s$, we can also perform training without knowing 
$\pi$ and $c$. 

\subsection{Extension to the Two-sample Scenario}
\label{ours_ts}

In the main paper, we assume the one-sample scenario for collecting PU data.
In this section, we consider the proposed method for the two-sample scenario.
In this scenario, PU data are sampled independently:
labeled positive data are drawn from $p^{{\rm l}} ({\bf x})$ and unlabeled data are drawn from true marginal density $p({\bf x})$. As a result,
we are given the following data for training:
\begin{align}
\label{given_data2}
X^{{\rm p}} &:= \{ {\bf x} ^{{\rm p}}_n \}_{n=1}^{N^{{\rm p}}} \sim p^{{\rm l}} ({\bf x}) := p({\bf x} | o=1), \nonumber\\
X &:= \{ {\bf x} _m \}_{m=1}^N \sim p ({\bf x}), \nonumber\\
R^{{\rm p}} &:= \{ r^{{\rm p}}_n | \ r^{{\rm p}}_n = p(y=1 | {\bf x} ^{{\rm p}}_n) \}_{n=1}^{N^{{\rm p}}}.
\end{align}
In this scenario, the relationship between $p^{{\rm l}} ({\bf x})$ and $p^{{\rm p}} ({\bf x})$ in Eq. \eqref{rel_po} is also satisfied. Additionally, propensity score $e({\bf x})$ can be rewritten as
\begin{align}
\label{ex_ts}
e({\bf x}) &= p(o=1|{\bf x}, y=1) \nonumber\\
&= \frac{p(o=1, {\bf x}, y=1)}{p({\bf x}, y=1)} \nonumber\\
&= \frac{p(y=1|{\bf x}, o=1)p({\bf x}|o=1)p(o=1)}{p(y=1|{\bf x})p({\bf x})} \nonumber\\
&= \frac{p({\bf x}|o=1)p(o=1)}{p(y=1|{\bf x})p({\bf x})} \nonumber\\
&= \frac{\alpha w({\bf x})}{p(y=1|{\bf x})},
\end{align}
where we used the Bayes' rule in the second equality, the PU property (labeled data are always positive), $p(y=1|{\bf x}, o=1)=1$, in the fourth equality, and set $w({\bf x}):= \frac{p({\bf x}|o=1)}{p({\bf x})}$ and $\alpha := p(o=1)$ in the fifth equality.
Here, unlike Eq. \eqref{e_sy} of the one-sample scenario, we do not use $p(o=1|{\bf x})$ to represent $e({\bf x})$.
This is because the numbers of PU data are arbitrarily determined by the users (data collectors) in the two-sample scenario and thus true label ratio $\alpha$ cannot be estimated from the given data
\citep{elkan2008learning}.
As a result, $p(o=1|{\bf x}) = \alpha p({\bf x}|o=1)/p({\bf x})$ cannot also be estimated from the given data
\citep{elkan2008learning}.
In contrast, density-ratio $w({\bf x}) = p({\bf x}|o=1)/p({\bf x})$ in Eq. \eqref{ex_ts} can be estimated from given PU data
without knowing $\alpha$
by using off-the-shelf density-ratio estimation methods
\citep{kanamori2009least,kato2021non,kato2018learning,sugiyama2012density}.

By substituting Eqs. \eqref{rel_po} and \eqref{ex_ts} into Eq. \eqref{obj_auc_won_c}, 
AUC risk $R_{\sigma}(s)$ can be represented as
\begin{align}
\label{obj_auc_ours_ts}
{\cal R}_{\sigma} (s) &= \frac{c}{(1-\pi)\alpha}\mathbb{E}_{{\bf x}^{{\rm p}} \sim p^{{\rm l}} ({\bf x})} \mathbb{E}_{{\bf x} \sim p ({\bf x} )} \left[ \frac{p(y=1|{\bf x}^{{\rm p}})}{w({\bf x} ^{{\rm p}})} f ({\bf x}^{{\rm p}}, {\bf x} ) \right] -\frac{\pi}{2(1-\pi)}.
\end{align}
Here, coefficient $c/((1-\pi)\alpha)$ and constant $\pi/(2(1-\pi))$ do not affect the optimization for score function $s$ and thus we can safely ignore them. The estimator of the AUC risk in Eq. \eqref{obj_auc_ours_ts} with training data $X^{{\rm p}} \cup X \cup R^{{\rm p}}$ can be represented as
\begin{align}
\label{obj_auc_ours2_emp_ts}
{\hat {\cal R}}_{\sigma} (s) &= \frac{c}{(1-\pi)\alpha} \left[ \frac{1}{N^{{\rm p}}N} \sum_{n, m=1}^{N^{{\rm p}},N} \frac{r_n^{{\rm p}}}{\hat{w} _n^{{\rm p}}} f ({\bf x}_n^{{\rm p}}, {\bf x}_m )  \right] -\frac{\pi}{2(1-\pi)},
\end{align}
where $\hat{w}^{{\rm p}}_n$ denotes the estimate of density-ratio $w({\bf x}_n^{{\rm p}})$ for instance ${\bf x}_n^{{\rm p}}$.

\section{Additional Experimental Details}

\subsection{Datasets and Data Construction}
\label{apen:dataset}

Mnist consists of hand-written images of 10 digits.
Fmnist consists of images of 10 fashion categories.
In Mnist and Fmnist, each image is represented by gray scale with $28 \times 28$ pixels.
Svhn consists of images of 10 printed digits clipped from photographs of house number plates. 
Cifar10 consists of images of 10 animal and vehicle categories.
In Svhn and Cifar10, each image is represented by $32 \times 32$ RGB pixels.
Diabetes is a tabular dataset whose task is to predict whether the respondent has diabetes or not.
Each respondent is represented as $142$-dimensional features.
Blood is a tabular dataset whose task is hypertension diagnosis for high-risk age people. 
Each survey subject is represented as $100$-dimensional features.

We created the binary classification problems from each dataset following the data creation procedure of previous studies
\citep{kumagai2025importance,kumagaiauc,kumagaipositive,xie2024weakly}.
Specifically, for Mnist and Svhn, we used odd digits as positive and even digits as negative.
For Fmnist, 
we used non-upper garment classes (1, 5, 7, 8, and 9) as positive and the others as negative.
For Cifar10, we used the vehicle classes (0, 1, 8, and 9) as positive and the animal classes as negative.
For Diabetes and Blood, we used the original binary class labels.

To create the sampling bias, for Mnist and Svhn,
90$\%$ of data were selected from digits 1, 3, and 5, and 10$\%$ of data were selected from digits 7 and 9 for labeled positive data.
For Fmnist, 90$\%$ of data were selected from classes 1, 5, and 7, and 10$\%$ of data were selected from classes 8 and 9 for labeled positive data.
For Cifar10, 90$\%$ of data were selected from classes 0, and 1, and 10$\%$ of data were selected from classes 8 and 9 for labeled positive data.
The remaining unselected data were used for unlabeled data for training and validation.

To create the selection bias in Diabetes and Blood datasets, we followed the procedure used in previous studies
\citep{sakai2019covariate,kumagaipositive}.
Specifically, we first calculated the Euclidean distance between positive instance ${\bf x} _n$ and the mean vector of all positive data ${\bar {\bf x}}$, ${\bf c} _n := \| {\bf x}_n - {\bar {\bf x}} \|$.
Then, we found the median ${\bf c}_{{\rm med}}$ from all $\{ {\bf c} _n \}$.
We split all $\{ {\bf c} _n \}$ into the first set whose elements were smaller than ${\bf c}_{{\rm med}}$ and the second set whose elements were greater than or equal to ${\bf c}_{{\rm med}}$. 
Within labeled positive data, $90\%$ data were selected from the first set and $10\%$ data from the second set.
In addition, we provide the results when using other selection biases for the image and tabular datasets in Section \ref{sec_insdep_bias}.

\subsection{Datasets with Human-derived Confidence}
\label{apen:human_data}

Cifar10-H and Fmnist-H were real-world image datasets where each instance was 
labeled by approximately $50$ and $67$ real-world human annotators, respectively
\citep{peterson2019human,ishidaperformance}.
In these datasets, confidence information for each instance can be obtained by averaging the corresponding annotations
\citep{ishidaperformance}.
For Cifar10-H,
we treated the land-related classes (1, 3, 4, 5, 7, and 9) as positive and the remaining classes (e.g., water- or sky-related classes) as negative following the previous study
\citep{ishidaperformance}. 
To create the sampling bias, 90$\%$ of the labeled positive data were sampled from classes 1, 3, and 4.
For Fmnist-H,
we treated the casual classes (0, 1, 2, 6, and 7) as positive and the non-casual classes (3, 4, 5, 8, and 9) as negative.
To create the sampling bias, 90$\%$ of the labeled positive data were sampled from classes 0, 1, and 2. 
Since the amount of annotated data was limited, we set the number of initially sampled unlabeled training to $3,000$ for both datasets.
The other conditions were the same as those used for Cifar10 and Fmnist.

\subsection{Model and Training Details}
\label{arc}

For Mnist, Fmnist, Fmnist-H, Diabetes, and Blood, 
a three-layered feed-forward neural network with ReLU activation was used for the classifier (score function).
The number of hidden nodes was $124$.
For Svhn, Cifar10, and Cifar10-H, a convolutional neural network, which consisted of two convolutional blocks followed by a three-layered feed-forward neural network, was used for the classifier (score function). The first (second) convolutional block comprised a 6 (16) filter $5 \times 5$ convolution, the ReLU activation, and a $2 \times 2$ max-pooling layer.
For all comparison methods including the classifier to estimate positive-confidence, the same neural network architecture was used.

For the proposed method, w/oConf, nnPU, PUAUC, Pconf, and NPU, the sigmoid function was used as a surrogate, following the previous studies
\citep{kiryo2017positive,kumagaiauc,kumagaipositive}.
For PUSB, we used the logarithmic loss function to deal with the selection bias following the original paper
\citep{kato2018learning}.
For NTC and the classifier to estimate $p(y=1|{\bf x})$, we used the logistic regression with the neural network as probabilistic classifiers.
For nnPU, PUAUC, PUSB, and NPU, whose original objectives involve expectations over the marginal density, we exactly reexpressed these expectations in terms of the labeled positive and unlabeled densities under the one-sample sampling scheme~\citep{bekker2020learning} for fair comparisons.

For all methods, the empirical risk (loss) with validation PU data was used for early-stopping to mitigate overfitting.
For NPU, the weighting parameter was also chosen on the basis of the validation empirical risk from $\{0.1, 0.2, 0.3, 0.4, 0.5, 0.6, 0.7, 0.8, 0.9 \}$.
Confidence score $\tilde{r}({\bf x})$ used in NPU was set to one since the observed positive labels are noise-free in our setting. 
For the proposed method and w/oConf, we rounded up $\hat{u}_n^{{\rm p}} = \hat{u}({\bf x} _n^{{\rm p}})$ in Eq. \eqref{obj_auc_ours2_emp} 
less than $0.01$ to $0.01$ to stabilize the
training process following the previous study
\citep{ishida2018binary}.
Similarly, for Pconf and NPU, we rounded up $r_n^{{\rm p}} = p(y=1|{\bf x}_n^{{\rm p}})$
less than $0.01$ to $0.01$ for their objectives.
For all methods, we used the Adam optimizer
\citep{kingma2014adam}.
We set the learning rate to $10^{-4}$. 
Total mini-batch size $M = U + P$ was set to $1,024$ and PU mini-batch sizes $P$ and $U$ were set so as to maintain the label ratio $\alpha$ in the given all data.
For the classifier for estimating $p(y=1|{\bf x})$, we set the training epoch to $30$.
For other methods, the maximum number of epochs was $200$ and the early-stopping was used.
All methods were implemented using PyTorch
\citep{paszke2017automatic},
and all experiments were conducted on a Linux server with an Intel Xeon CPU and A100 GPU.

\section{Additional Experimental Results}

\subsection{Results with Different Positive Class-priors $\pi$}
\label{pi_graph}

Figure \ref{fig0} shows the average test AUCs with their standard errors when changing the value of positive class-prior $\pi$.
The proposed method tended to perform well with each $\pi$ value. 
Most methods tended to improve the performance as the value of $\pi$ increased.
This is because the number of labeled positive data increased when $\pi$ increased
due to the relationship $\alpha = p(o=1) = c\pi$ with fixed $c=p(o=1|y=1)$.

Figure \ref{fig1} shows the average test AUCs with their standard errors when changing the value of positive class-prior $\pi$ while maintaining the labeled positive data size.
To maintain the size, we adjusted the value of $c=p(o=1|y=1)$ such that $p(o=1)=c\pi=0.02$ for each $\pi$.
As a result, the labeled positive data and unlabeled data sizes were $100$ and $4,900$, respectively.
The proposed method worked well in most cases and simply increasing $\pi$ did not lead to performance improving.

\begin{figure*}[t]
  \centering
  \begin{subfigure}{0.31\textwidth}
    \centering
    \includegraphics[width=\linewidth]{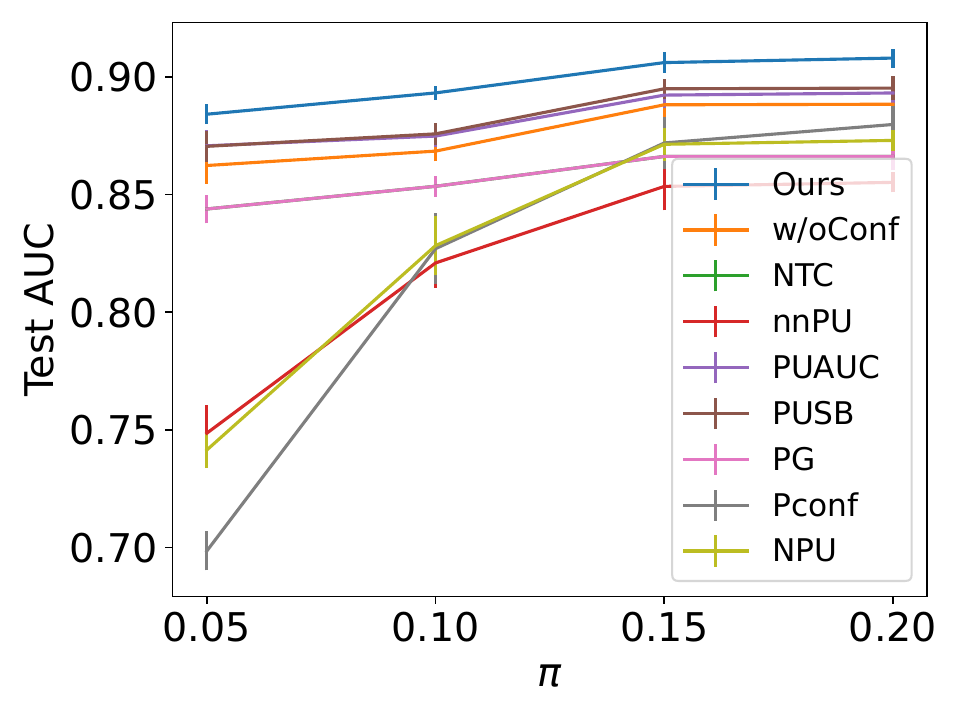}
    \caption{Mnist}
  \end{subfigure}
  \begin{subfigure}{0.31\textwidth}
    \centering
    \includegraphics[width=\linewidth]{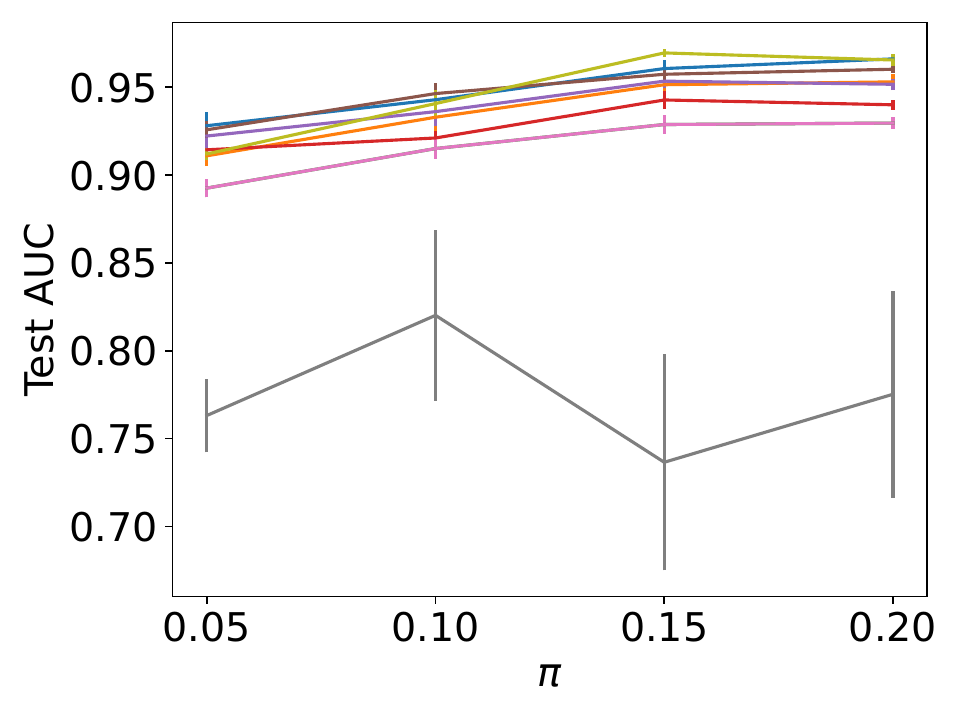}
    \caption{Fmnist}
  \end{subfigure}
  \begin{subfigure}{0.31\textwidth}
    \centering
    \includegraphics[width=\linewidth]{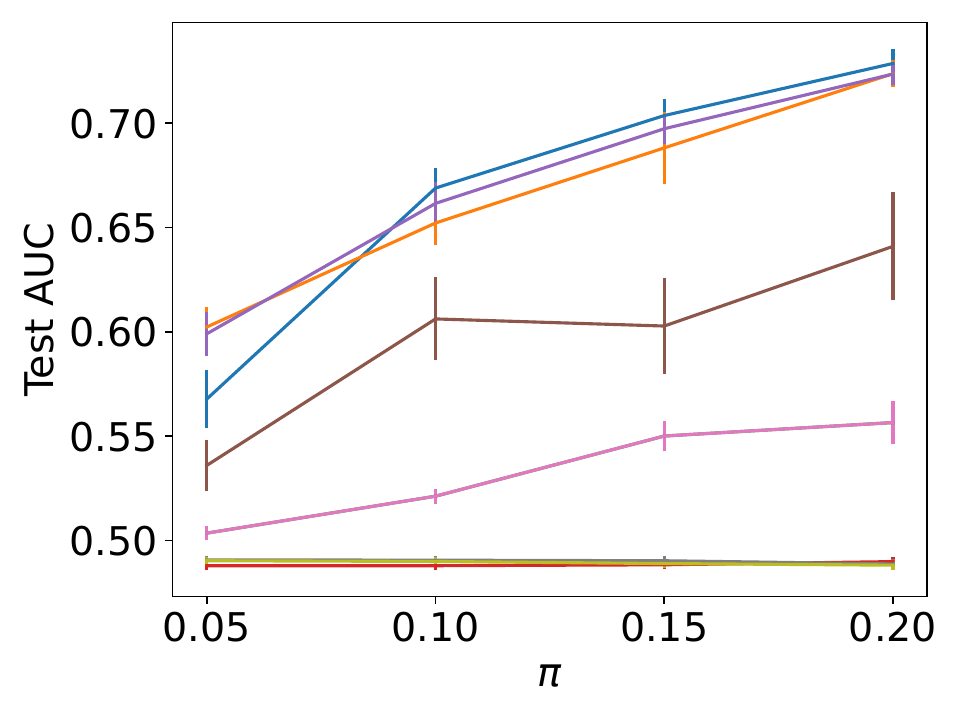}
    \caption{Svhn}
  \end{subfigure}
  \begin{subfigure}{0.31\textwidth}
    \centering
    \includegraphics[width=\linewidth]{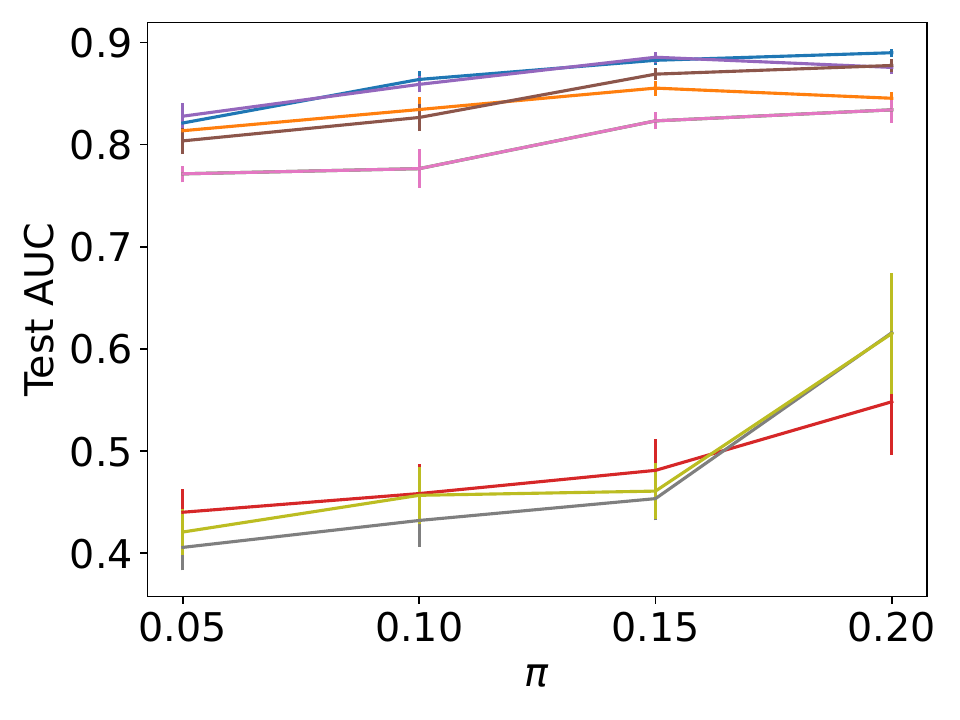}
    \caption{Cifar10}
  \end{subfigure}
  \begin{subfigure}{0.31\textwidth}
    \centering
    \includegraphics[width=\linewidth]{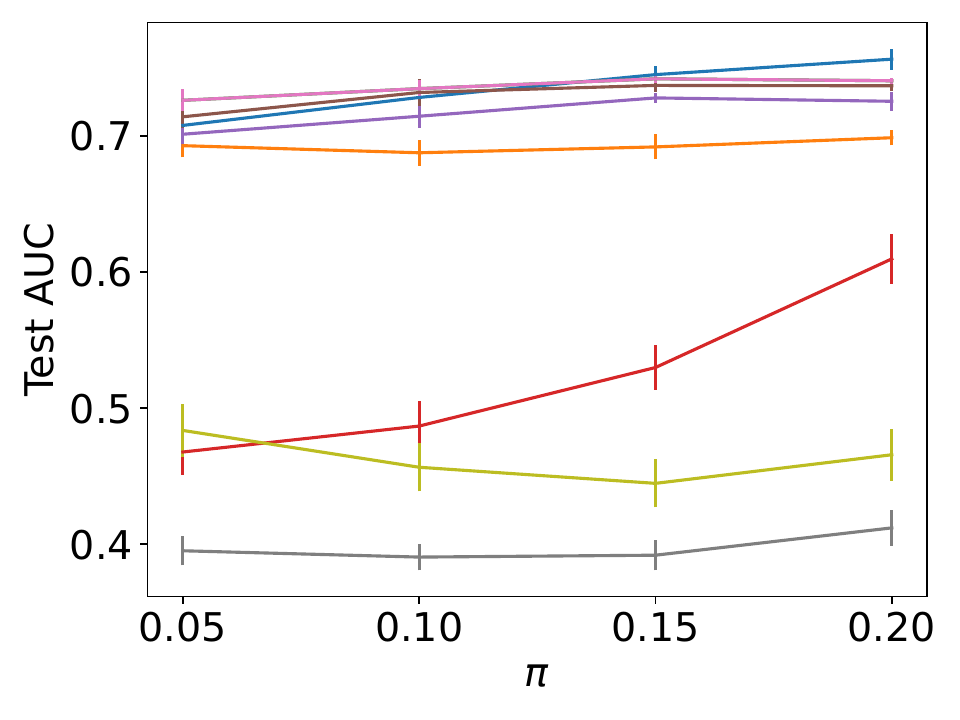}
    \caption{Diabetes}
  \end{subfigure}
  \begin{subfigure}{0.31\textwidth}
    \centering
    \includegraphics[width=\linewidth]{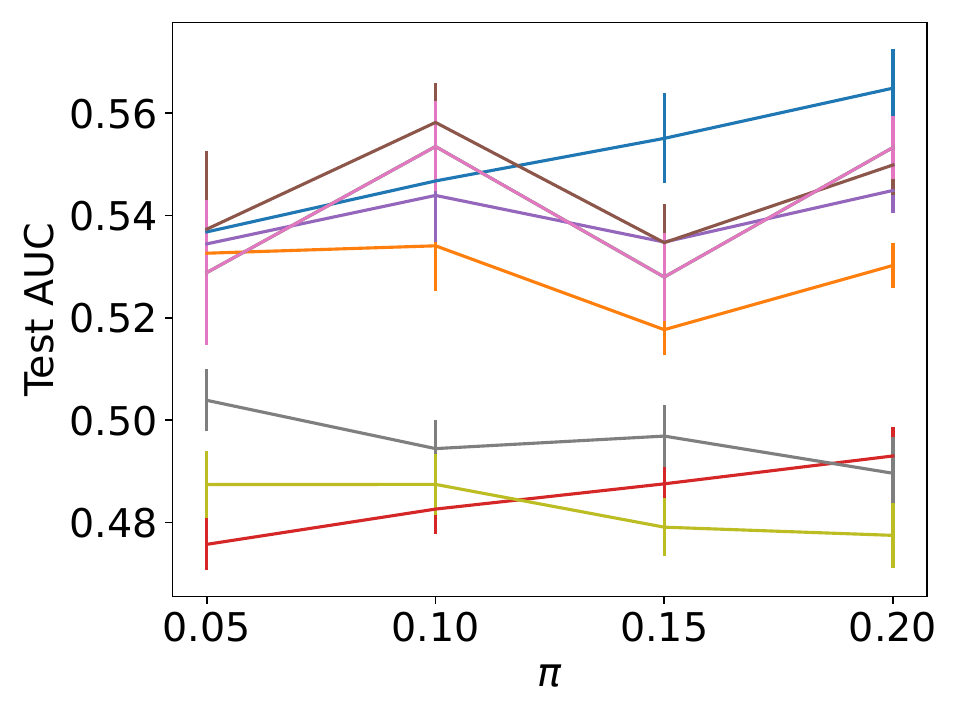}
    \caption{Blood}
  \end{subfigure}
  \caption{Average test AUCs with their standard errors when changing positive class-prior $\pi$ within $\{0.05, 0.1, 0.15, 0.2\}$.}
  \label{fig0}
\end{figure*}

\begin{figure*}[t]
  \centering
  \begin{subfigure}{0.3\textwidth}
    \centering
    \includegraphics[width=\linewidth]{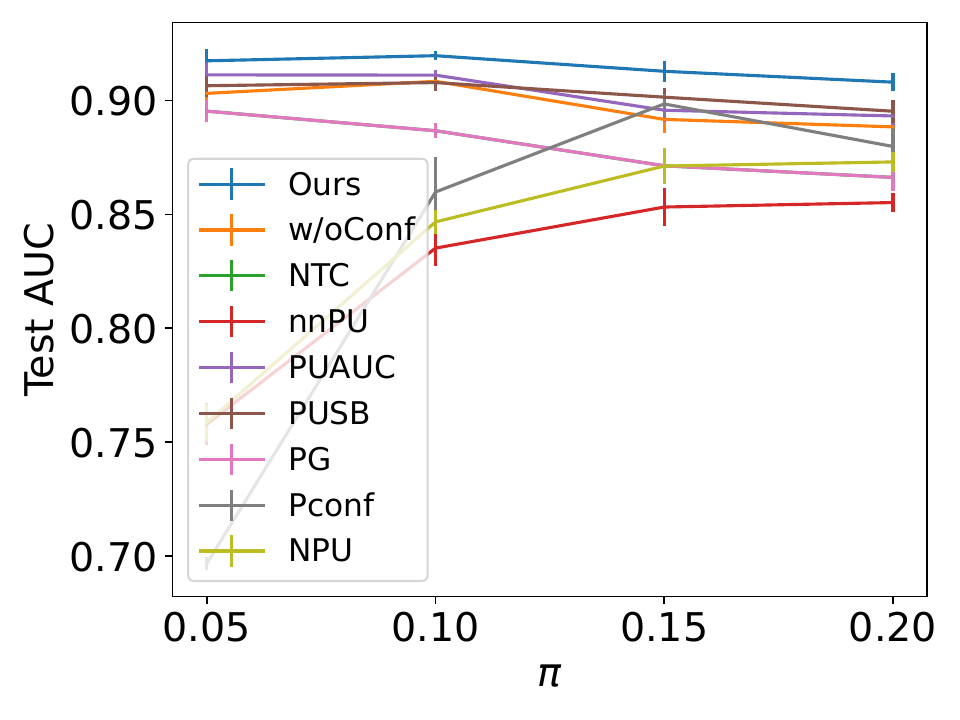}
    \caption{Mnist}
  \end{subfigure}
  \begin{subfigure}{0.3\textwidth}
    \centering
    \includegraphics[width=\linewidth]{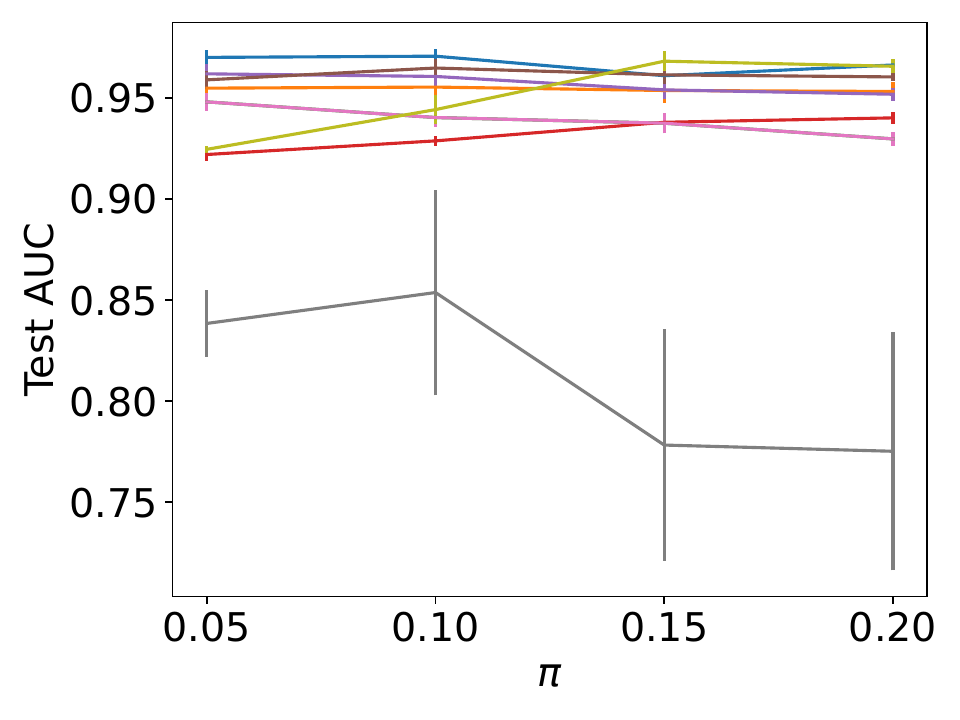}
    \caption{Fmnist}
  \end{subfigure}
  \begin{subfigure}{0.3\textwidth}
    \centering
    \includegraphics[width=\linewidth]{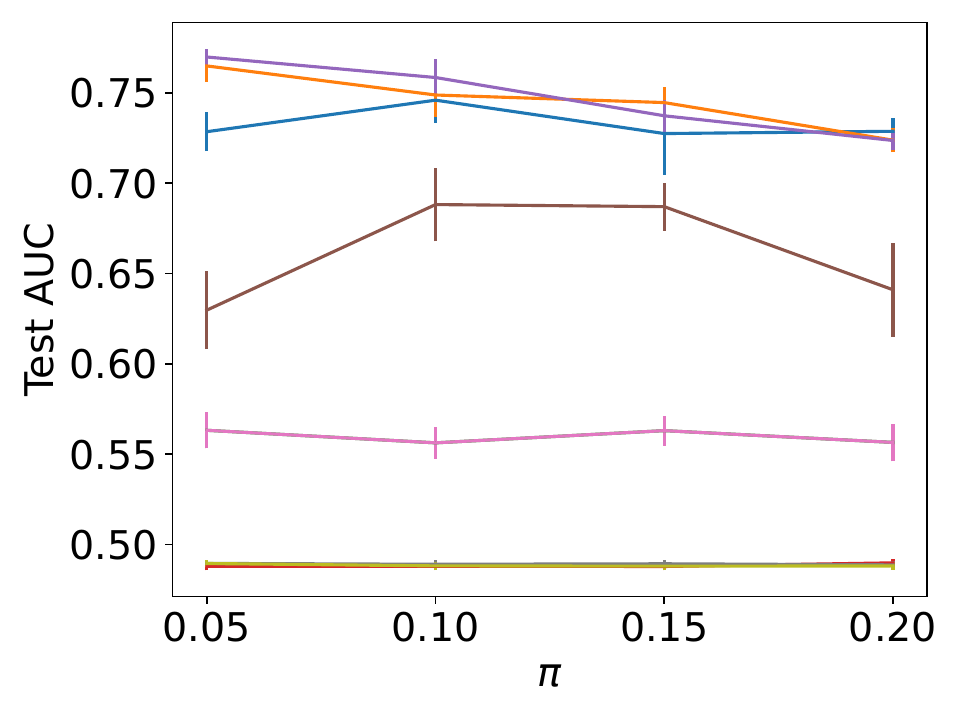}
    \caption{Svhn}
  \end{subfigure}
  \begin{subfigure}{0.3\textwidth}
    \centering
    \includegraphics[width=\linewidth]{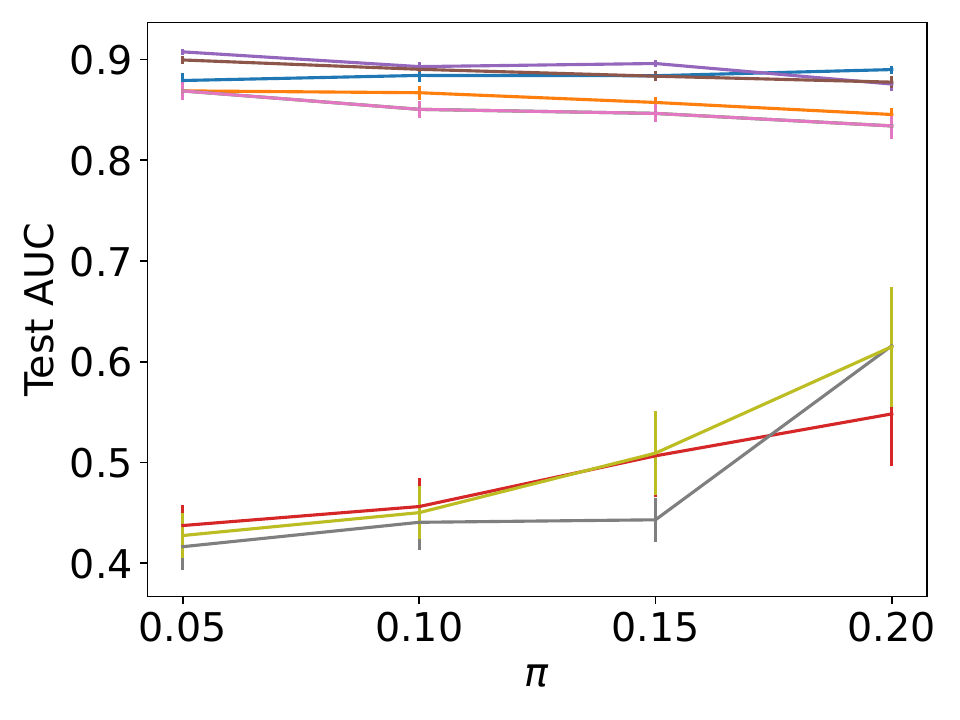}
    \caption{Cifar10}
  \end{subfigure}
  \begin{subfigure}{0.3\textwidth}
    \centering
    \includegraphics[width=\linewidth]{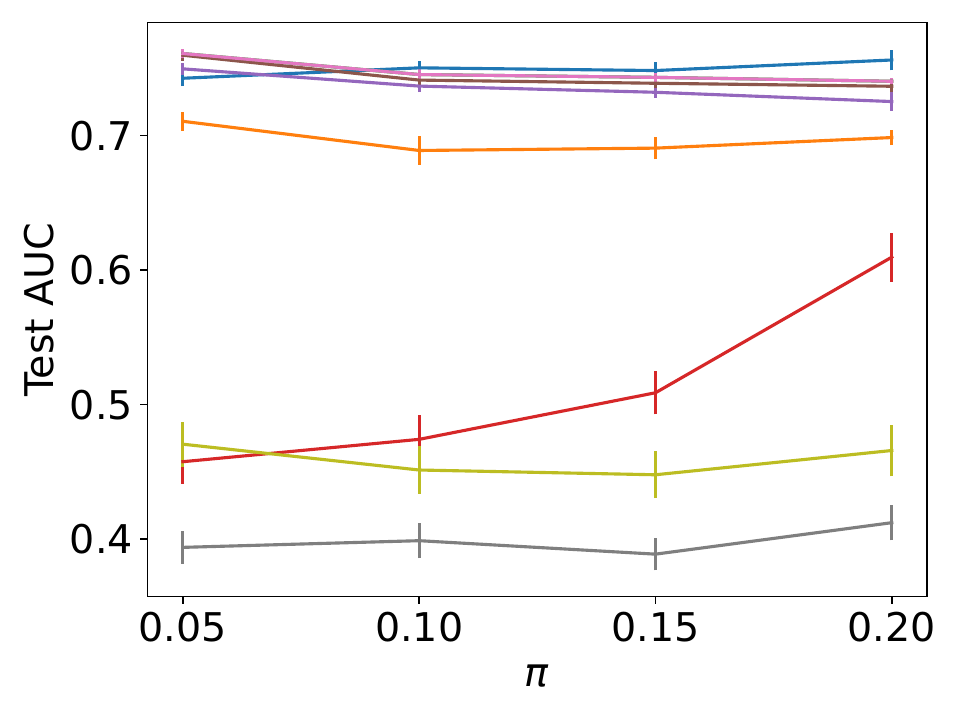}
    \caption{Diabetes}
  \end{subfigure}
  \begin{subfigure}{0.3\textwidth}
    \centering
    \includegraphics[width=\linewidth]{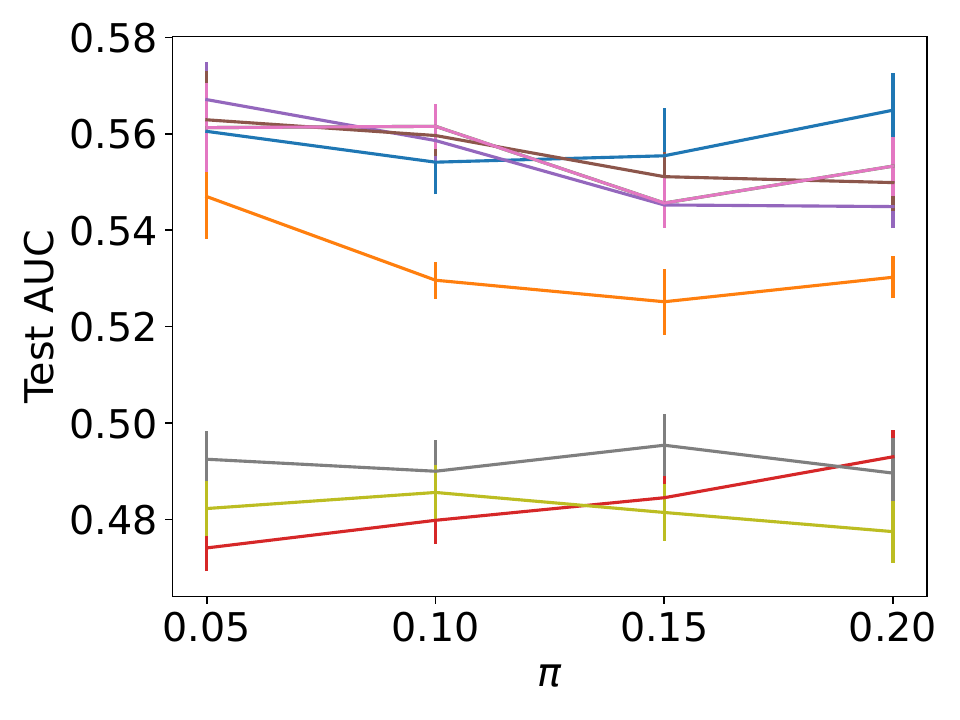}
    \caption{Blood}
  \end{subfigure}
  \caption{Average test AUCs with their standard errors when changing positive class-prior $\pi$ within $\{0.05, 0.1, 0.15, 0.2\}$ with the fixed number of labeled positive data.}
  \label{fig1}
\end{figure*}

\subsection{Results with Different Positive Labeling Rates $c=p(o=1|y=1)$}
\label{dif_c_result}

Tables \ref{tab:c005} and \ref{tab:c015} show the average test AUCs over different positive class-prior $\pi$ when changing the value of $c=p(o=1|y=1)$.
The proposed method performed well for each $c$ value.
As $c$ increases, the number of labeled positive data increases, and thus the performance of the proposed method tends to improve.

\begin{table*}[t]
\caption{Average test AUCs over different positive class-prior $\pi$ within $\{0.05, 0.1, 0.15, 0.2\}$ with $c=0.05$.}
\label{tab:c005}
\centering
\scalebox{0.95}{
\begin{tabular}{lrrrrrrrrr}
\hline
Data 
& \multicolumn{1}{c}{Ours} 
& \multicolumn{1}{c}{w/oConf} 
& \multicolumn{1}{c}{NTC} 
& \multicolumn{1}{c}{nnPU} 
& \multicolumn{1}{c}{PUAUC} 
& \multicolumn{1}{c}{PUSB} 
& \multicolumn{1}{c}{PG} 
& \multicolumn{1}{c}{Pconf} 
& \multicolumn{1}{c}{NPU} \\
\hline
Mnist    
& \textbf{0.8643} & 0.8465 & 0.8136 & 0.8009 & 0.8538 & 0.8525 & 0.8136 & 0.7697 & 0.8095 \\
Fmnist   
& \textbf{0.9220} & 0.9141 & 0.8827 & 0.9205 & 0.9164 & \textbf{0.9305} & 0.8827 & 0.6691 & \textbf{0.9318} \\
Svhn     
& \textbf{0.5896} & \textbf{0.5826} & 0.5090 & 0.4885 & \textbf{0.5911} & 0.5465 & 0.5090 & 0.4906 & 0.4898 \\
Cifar10  
& \textbf{0.8235} & 0.7974 & 0.7312 & 0.4751 & \textbf{0.8151} & 0.7979 & 0.7312 & 0.4976 & 0.4838 \\
Diabetes 
& \textbf{0.7109} & 0.6821 & 0.6949 & 0.5660 & 0.6905 & \textbf{0.7065} & 0.6949 & 0.4013 & 0.4744 \\
Blood    
& \textbf{0.5311} & 0.5170 & 0.5144 & 0.4905 & 0.5196 & \textbf{0.5269} & 0.5144 & 0.5041 & 0.4780 \\
\hline
\# Best  
& 6 & 1 & 0 & 0 & 2 & 3 & 0 & 0 & 1 \\
\hline
\end{tabular}
}
\end{table*}
\begin{table*}[t!]
\caption{Average test AUCs over different positive class-prior $\pi$ within $\{0.05, 0.1, 0.15, 0.2\}$ with $c=0.15$.}
\label{tab:c015}
\centering
\scalebox{0.95}{
\begin{tabular}{lrrrrrrrrr}
\hline
Data 
& \multicolumn{1}{c}{Ours} 
& \multicolumn{1}{c}{w/oConf} 
& \multicolumn{1}{c}{NTC} 
& \multicolumn{1}{c}{nnPU} 
& \multicolumn{1}{c}{PUAUC} 
& \multicolumn{1}{c}{PUSB} 
& \multicolumn{1}{c}{PG} 
& \multicolumn{1}{c}{Pconf} 
& \multicolumn{1}{c}{NPU} \\
\hline
Mnist    
& \textbf{0.9131} & 0.8931 & 0.8712 & 0.8216 & 0.8978 & 0.8972 & 0.8712 & 0.8456 & 0.8323 \\
Fmnist   
& \textbf{0.9589} & 0.9475 & 0.9333 & 0.9317 & 0.9512 & 0.9544 & 0.9333 & 0.8156 & \textbf{0.9529} \\
Svhn     
& \textbf{0.7095} & \textbf{0.7185} & 0.5524 & 0.4882 & \textbf{0.7167} & 0.6493 & 0.5524 & 0.4891 & 0.4884 \\
Cifar10  
& 0.8804 & 0.8548 & 0.8350 & 0.4949 & \textbf{0.8882} & 0.8751 & 0.8350 & 0.4867 & 0.4902 \\
Diabetes 
& \textbf{0.7478} & 0.7000 & \textbf{0.7483} & 0.5032 & 0.7321 & \textbf{0.7439} & \textbf{0.7483} & 0.3956 & 0.4628 \\
Blood    
& \textbf{0.5609} & 0.5314 & 0.5489 & 0.4829 & 0.5431 & 0.5495 & 0.5489 & 0.4947 & 0.4826 \\
\hline
\# Best  
& 5 & 1 & 1 & 0 & 2 & 1 & 1 & 0 & 1 \\
\hline
\end{tabular}
}
\end{table*}

\begin{figure}[t!]
\centering
\includegraphics[width=7.0cm]{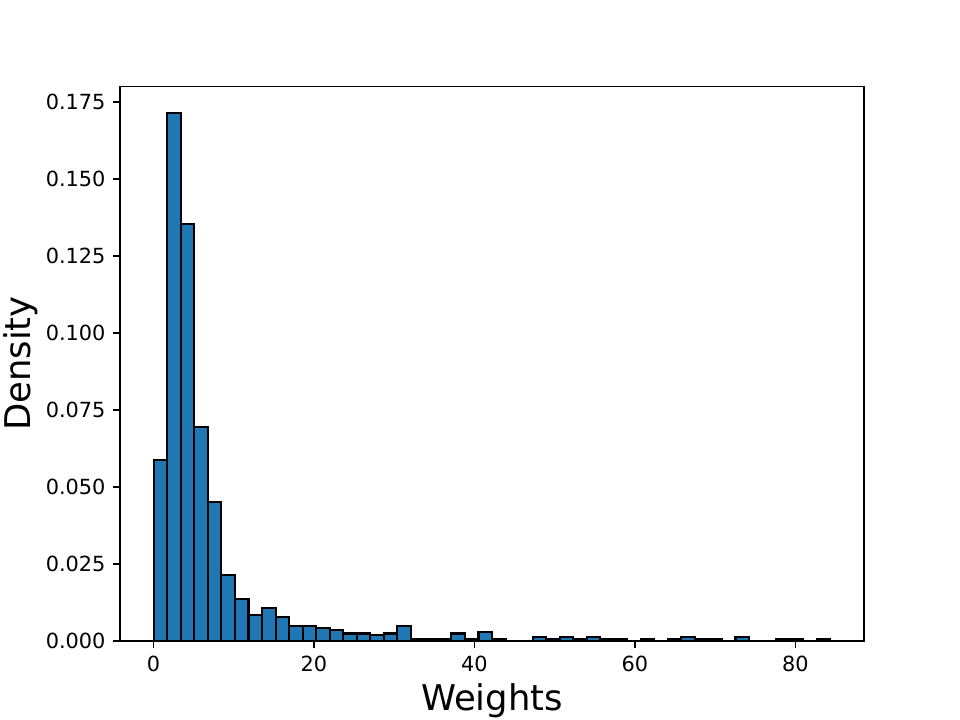}
\caption{Weight distribution of the proposed method with the Mnist dataset when $\pi=0.2$.}
\label{vis_weight}
\end{figure}

\subsection{Visualization}

The proposed method assigns weight $p(y=1|{\bf x})/p(o=1|{\bf x})$ to the labeled positive instance ${\bf x}$ in the AUC risk of Eq. \eqref{obj_auc_ours2} for dealing with the bias of given positive data.
In this section, we investigated the distribution of the estimated weights for labeled positive data with the Mnist dataset.
Figure \ref{vis_weight} shows the result.
As can be seen, the values of the weights ranged from around $0.05$ to about $84.0$, with the majority distributed between $0.0$ and $10.0$. By leveraging these estimated weights, the proposed method performed superiorly.

\subsection{Results with Confidence-dependent Selection Bias}
\label{sec_insdep_bias}

In the main paper, we considered class-based and distance-based selection biases for the image and tabular datasets, respectively. In this section, we evaluated the proposed method using a confidence-dependent selection bias, which often arises in real-world applications. Specifically, 
we fixed the fraction of labeled positive instances at $c=0.1$, so that $N^{{\rm p}}=Vc\pi$, and sampled these instances from the positive instances in an initial data pool of size $V$, using sampling weights proportional to $\sigma(4r({\bf x})-2)$, where $r({\bf x})=p(y=1|{\bf x})$.
We used the probabilistic classifier for estimating $r({\bf x})$.
In this setting, positive instances with higher confidence were more likely to be selected as labeled data.
Table~\ref{result_conf_dep} shows the results.
The proposed method performed well under this confidence-dependent selection bias.

\begin{table*}[t]
\caption{The results using confidence-dependent selection bias: average test AUCs over different positive class-prior $\pi$ within $\{0.05, 0.1, 0.15, 0.2\}$. 
Values in bold are not statistically different at the $5\%$ level from the best performing method in each row according to a paired t-test.
}
\label{result_conf_dep}
\centering
\scalebox{0.95}{
\begin{tabular}{lrrrrrrrrrr}
\hline
Data & \multicolumn{1}{c}{Ours} & \multicolumn{1}{c}{w/oConf} & \multicolumn{1}{c}{NTC} & \multicolumn{1}{c}{nnPU} & \multicolumn{1}{c}{PUAUC} &  \multicolumn{1}{c}{PUSB} & \multicolumn{1}{c}{PG} & \multicolumn{1}{c}{Pconf} & \multicolumn{1}{c}{NPU} \\
\hline
Mnist & \bf{0.9142} & 0.9063 & 0.8941 & 0.8876 & 0.9132 & \textbf{0.9162} & 0.8941 & 0.7540 & 0.8901  \\
Fmnist & \bf{0.9748} & 0.9728 & 0.9488 & 0.9237 & \textbf{0.9738} & \textbf{0.9709}  & 0.9488 & 0.4847 & 0.9494  \\
Svhn & \bf{0.6546} & \bf{0.6594} & 0.5290 & 0.4890 & \bf{0.6650} & 0.5831 & 0.5290 & 0.4899 & 0.4894  \\
Cifar10 & \bf{0.8614} & 0.8499 & 0.8103 & 0.5196 & \bf{0.8595} & 0.8444 & 0.8103 & 0.5608 & 0.5788  \\
Diabetes & \bf{0.7684} & 0.7395 & \textbf{0.7682} & 0.5923 & 0.7643 & \textbf{0.7699} & \textbf{0.7682} & 0.3808 & 0.5249  \\
Blood & \bf{0.6129} & 0.5846 & 0.5927 & 0.5011 & 0.6076 & \textbf{0.6085} & 0.5927 & 0.4769 & 0.4762  \\
\hline
\end{tabular}
}
\end{table*}

\subsection{Computation Costs}

We investigated the training time of the proposed method on the Mnist dataset. 
We used a Linux server with a 2.20GHz CPU.
Table \ref{table:comptime} shows the results.
Although the proposed method (Ours), w/oConf, and PG require the training of $p(o=1|{\bf x})$,
we omitted its training time here because we wanted to purely investigate the computation cost of our proposed loss function in Eq. \eqref{obj_auc_ours2_emp}.
There were no significant differences in training time among all methods except for PG.
The training time of PG is short since the classifier of PG can be obtained analytically from trained $p(o=1|{\bf x})$.
However, it had lower test AUCs than the proposed method as described in Table~\ref{result_all}.
The full training times, including the training time of $p(o=1|{\bf x})$, for the proposed method (Ours), w/oConf, and PG were $169.3620$, $169.4360$, and $85.1475$, respectively.
Although the proposed method had a longer training time than other methods due to the additional training of $p(o=1|{\bf x})$, it had superior test AUCs.

\begin{table*}[t!]
\caption{Training time [sec] of each method on the Mnist dataset.}
\label{table:comptime}
\centering
\scalebox{0.9}{
\begin{tabular}{lccccccccc}
\hline
 & Ours & w/oConf & NTC & nnPU & PUAUC & PUSB & PG & Pconf & NPU \\
\hline
& 84.5424 & 84.6170 & 84.8195 & 84.5650 & 83.3238 & 85.2143 & 0.3280 & 80.4764 & 84.9054 \\
\hline
\end{tabular}
}
\end{table*}

\begin{table*}[t!]
\caption{Comparison with LBE: average test AUCs over different positive class-prior $\pi$ within $\{0.05, 0.1, 0.15, 0.2\}$. We set $c=p(o=1|y=1)=0.1$. Values in bold are not statistically different at the $5\%$ level from the best performing method in each column according to a paired t-test.}
\label{table:lbe}
\centering
\scalebox{0.9}{
\begin{tabular}{lrrrrrr}
\hline
Method & Mnist & Fmnist & Svhn & Cifar10 & Diabetes & Blood \\
\hline
Ours & \bf{0.8979} & \bf{0.9495} & \bf{0.6673} & \bf{0.8643} & \bf{0.7342} & \bf{0.5509} \\
LBE  & 0.8631 & 0.9245 & 0.5668 & 0.8271 & \bf{0.7292} & 0.5399 \\
\hline
\end{tabular}
}
\end{table*}

\subsection{Comparison with an Additional SAR-specific Method}
\label{apen:lbe}

To further compare the proposed method with methods specifically designed for the SAR setting, 
we additionally evaluated LBE (LBE-MLP)
\citep{gong2021instance}, 
a representative PU learning method for both the SAR setting and the one-sample scenario.
This method estimates both $p(y=1|{\bf x})$ and $p(o=1|{\bf x}, y=1)$ from biased PU data using the EM algorithm.
Table \ref{table:lbe} shows the results.
The proposed method often outperformed LBE. This would be because LBE is not designed for AUC maximization and has difficulty identifying $p(o=1|{\bf x}, y=1)$ from biased PU data.
In contrast, our method is designed to maximize AUC and to avoid the identifiability issue by using positive-confidence.

\subsection{Results under Strictly Increasing Transformations of Positive-confidence}
\label{apend:strans_results}

Figures \ref{fig:full_strans_h1} and \ref{fig:full_strans_h2} show the average test AUCs with their standard errors when changing $k$ of the transformations $h(r)=r^k$ and $h(r)=r^k/(r^k+(1-r)^k)$, respectively.
For both transformation forms, the performance of the proposed method remained relatively stable across different $k$ values and showed no large degradation as $k$ moved away from $1$.
In contrast, the performances of Pconf that also uses positive-confidence substantially degraded for some values of $k$, reflecting its sensitivity to transformations of the confidence values.

\begin{figure*}[t]
  \centering
  \begin{subfigure}{0.3\textwidth}
    \centering
    \includegraphics[width=\linewidth]{pi4_strans_mnist_eta100_2.pdf}
    \caption{Mnist}
  \end{subfigure}
  \begin{subfigure}{0.3\textwidth}
    \centering
    \includegraphics[width=\linewidth]{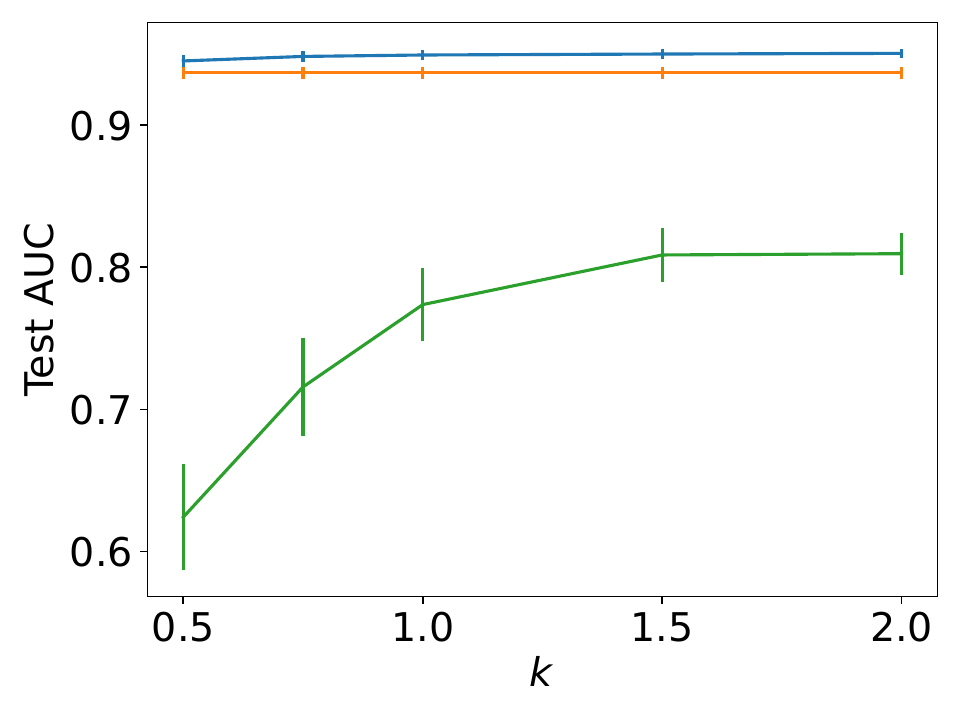}
    \caption{Fmnist}
  \end{subfigure}
  \begin{subfigure}{0.3\textwidth}
    \centering
    \includegraphics[width=\linewidth]{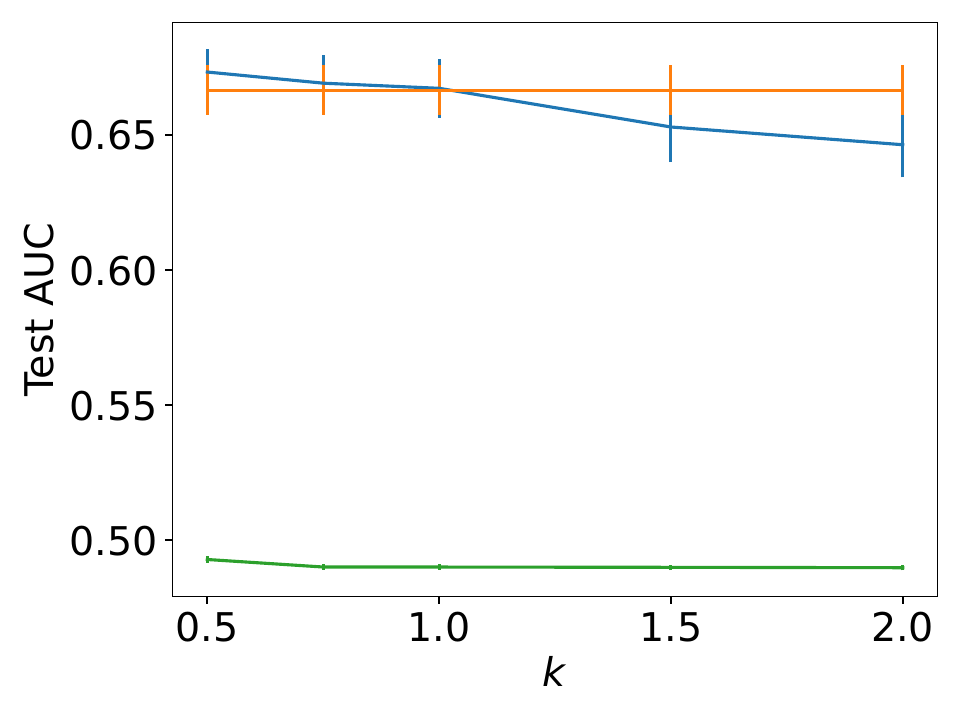}
    \caption{Svhn}
  \end{subfigure}
  \begin{subfigure}{0.3\textwidth}
    \centering
    \includegraphics[width=\linewidth]{pi4_strans_cifar10_eta100_2.pdf}
    \caption{Cifar10}
  \end{subfigure}
  \begin{subfigure}{0.3\textwidth}
    \centering
    \includegraphics[width=\linewidth]{pi4_strans_diabetes_eta100_2.pdf}
    \caption{Diabetes}
  \end{subfigure}
  \begin{subfigure}{0.3\textwidth}
    \centering
    \includegraphics[width=\linewidth]{pi4_strans_blood_eta100_2.pdf}
    \caption{Blood}
  \end{subfigure}
  \caption{The average test AUCs and their standard errors over different positive class-priors for different parameters $k$ of the transformation $h(r)=r^k$.}
  \label{fig:full_strans_h1}
\end{figure*}
\begin{figure*}[t!]
  \centering
  \begin{subfigure}{0.3\textwidth}
    \centering
    \includegraphics[width=\linewidth]{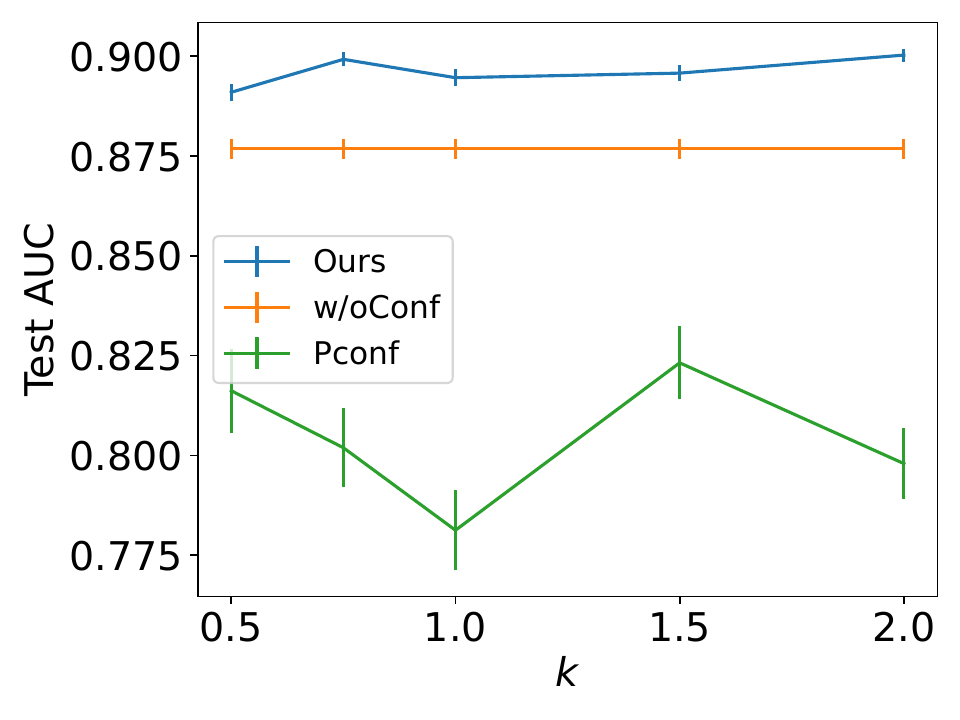}
    \caption{Mnist}
  \end{subfigure}
  \begin{subfigure}{0.3\textwidth}
    \centering
    \includegraphics[width=\linewidth]{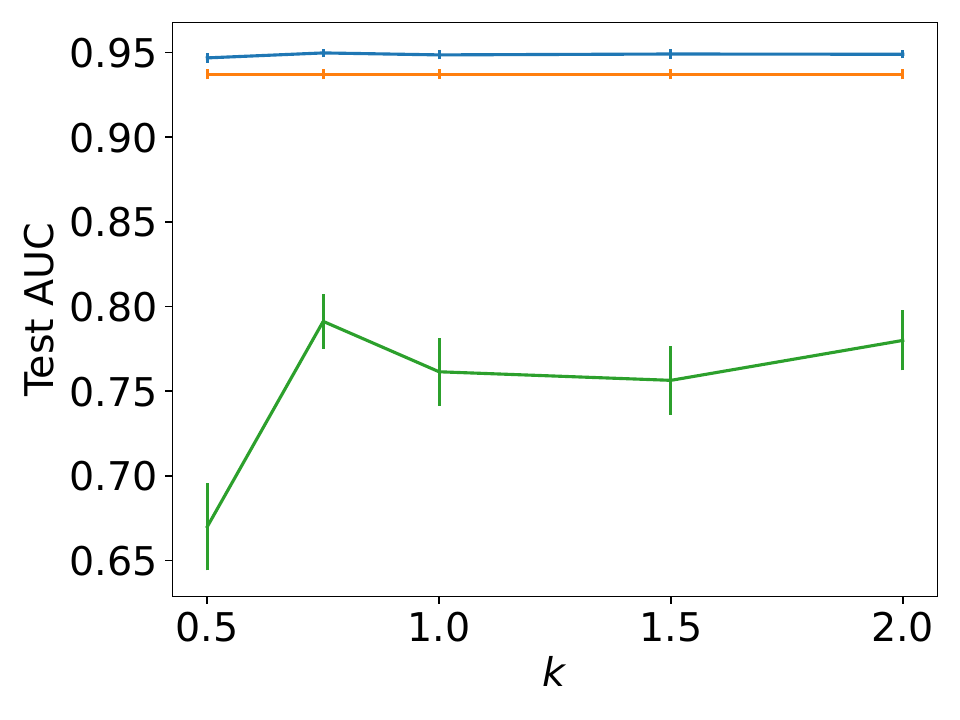}
    \caption{Fmnist}
  \end{subfigure}
  \begin{subfigure}{0.3\textwidth}
    \centering
    \includegraphics[width=\linewidth]{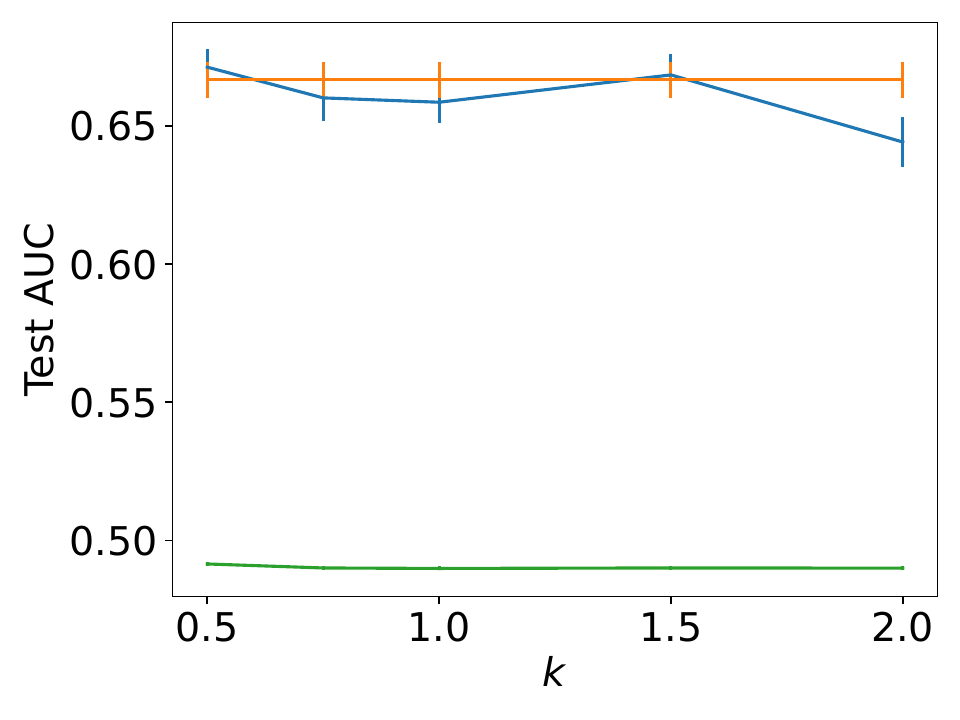}
    \caption{Svhn}
  \end{subfigure}
  \begin{subfigure}{0.3\textwidth}
    \centering
    \includegraphics[width=\linewidth]{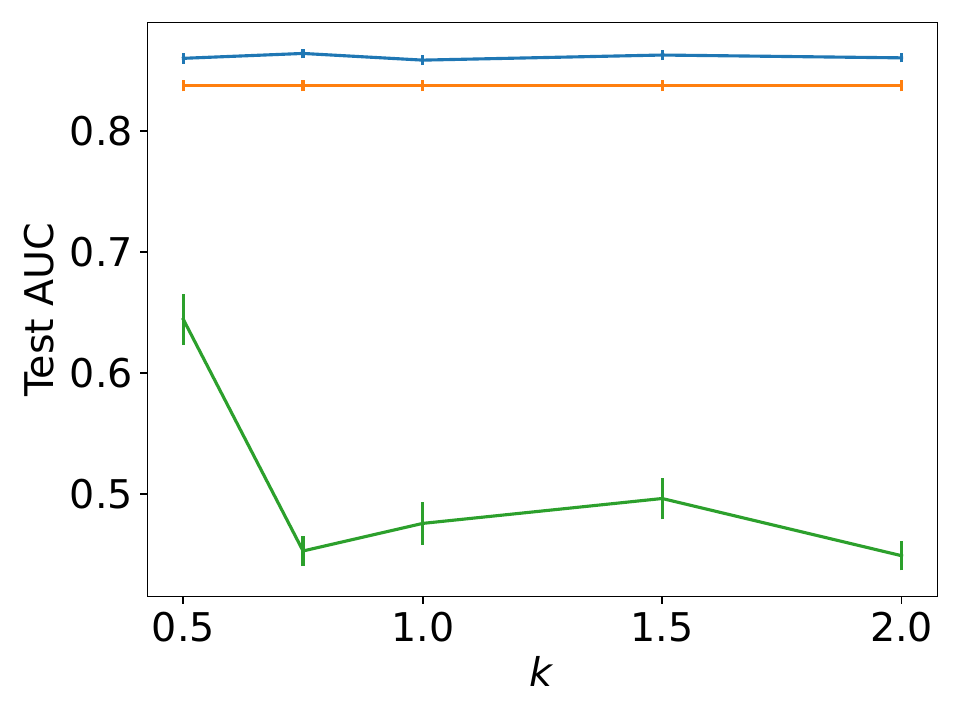}
    \caption{Cifar10}
  \end{subfigure}
  \begin{subfigure}{0.3\textwidth}
    \centering
    \includegraphics[width=\linewidth]{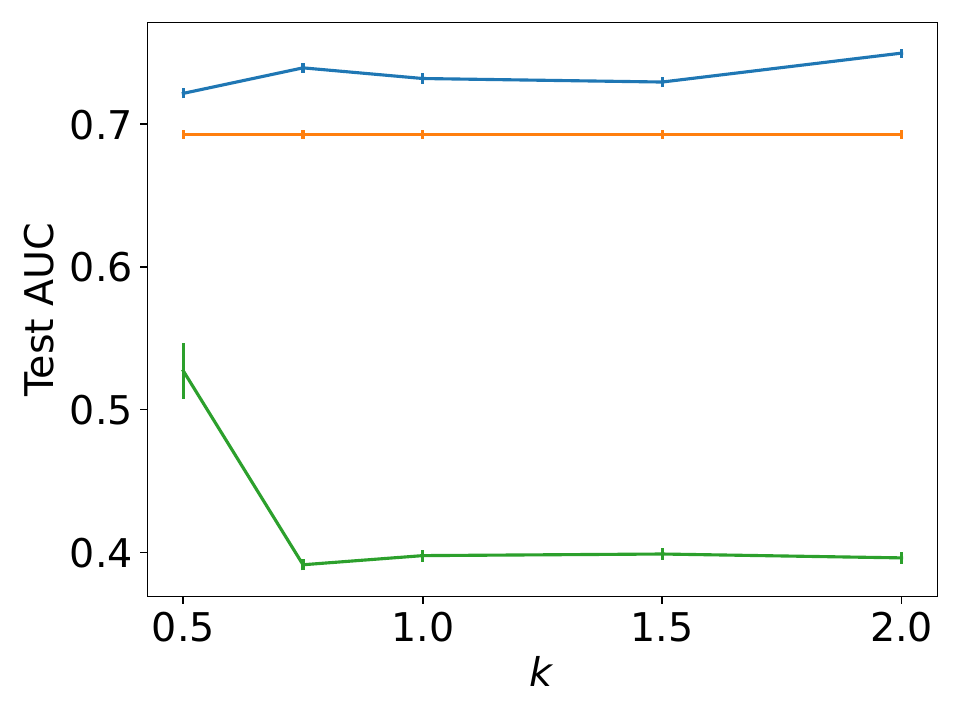}
    \caption{Diabetes}
  \end{subfigure}
  \begin{subfigure}{0.3\textwidth}
    \centering
    \includegraphics[width=\linewidth]{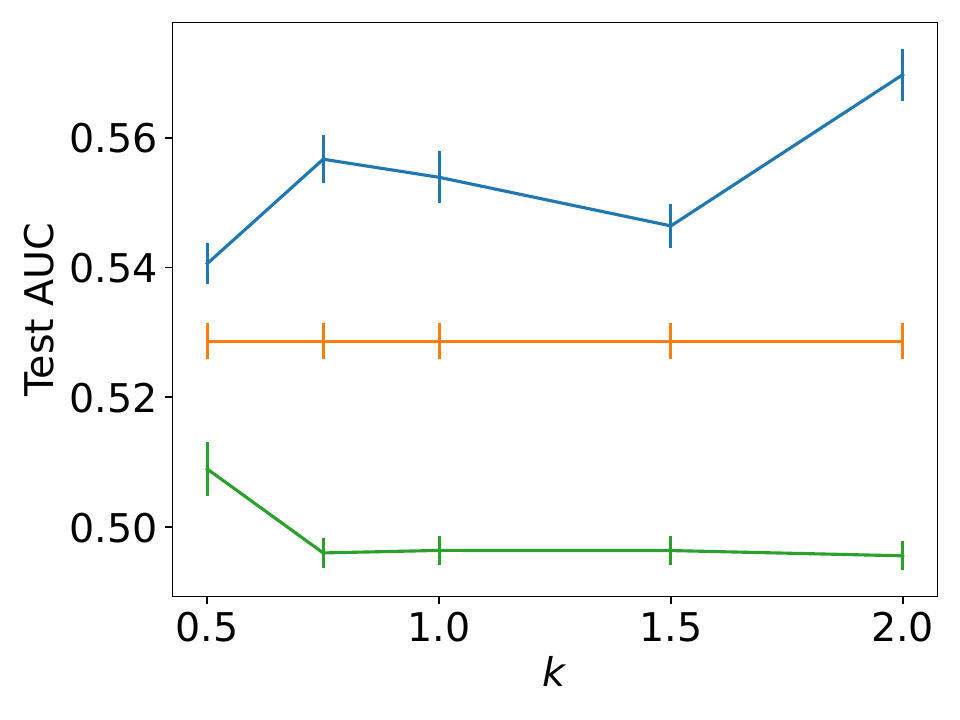}
    \caption{Blood}
  \end{subfigure}
  \caption{The average test AUCs and their standard errors over different positive class-priors for different parameters $k$ of the transformation $h(r)=r^k/(r^k+(1-r)^k)$.}
  \label{fig:full_strans_h2}
\end{figure*}

\subsection{Results with Gaussian Noise on Positive-confidence}
\label{apend:results_gnoise}

As shown in Section~\ref{robsut_pcon}, the proposed method can recover the Bayes-optimal AUC ranking even when the positive-confidence is distorted by a strictly increasing transformation.
We further evaluated its robustness when random noise was added to the positive-confidence.
For noisy positive-confidence, we added the zero-mean Gaussian noise with standard deviations chosen from $\{0.01, 0.05, 0.1, 0.15 \}$ following the previous study
\citep{ishida2018binary}.
As the standard deviation increases, the noise on positive-confidence becomes bigger.
When the modified positive-confidence was less than $0$ or greater than $1$, we clipped the value to $0$ and $1$, respectively.
Table \ref{result_noisy} shows the results.
As expected, the performance of the proposed method tended to decrease as the noise increased, because our objective function in Eq. \eqref{obj_auc_ours2_emp} explicitly depends on confidence values.
However, the degradation was relatively small on several datasets. Even with a standard deviation of $0.15$, the proposed method often outperformed w/oConf, which does not use positive-confidence. 
These results suggest that using positive-confidence can remain beneficial even when the confidence contains random noise.

\begin{table*}[t]
\caption{The results with noisy positive-confidence: average test AUCs over different positive class-prior $\pi$ within $\{0.05, 0.1, 0.15, 0.2\}$. Std represents the standard deviation of Gaussian noise. ${\rm Std}=0.0$ means that there is no noise.}
\label{result_noisy}
\centering
\scalebox{1.0}{
\begin{tabular}{lrrrrr|r}
\hline
Std & \multicolumn{1}{c}{0.0} & \multicolumn{1}{c}{0.01} & \multicolumn{1}{c}{0.05} & \multicolumn{1}{c}{0.1} & \multicolumn{1}{c}{0.15} & \multicolumn{1}{c}{w/oConf} \\
\hline
Mnist & 0.8979 & 0.8966 & 0.8932 & 0.8905 & 0.8883 & 0.8769  \\
Fmnist & 0.9495 & 0.9495 & 0.9501 & 0.9496 & 0.9489 & 0.9371 \\
Svhn & 0.6673 & 0.6662 & 0.6601 & 0.6367 & 0.6161 & 0.6666 \\
Cifar10 & 0.8643 & 0.8670 & 0.8623 & 0.8579 & 0.8552 & 0.8371 \\
Diabetes & 0.7342 & 0.7333 & 0.7277 & 0.7189 & 0.7013 & 0.6926 \\
Blood & 0.5509 & 0.5457 & 0.5435 & 0.5376 & 0.5321 & 0.5287 \\
\hline
\end{tabular}
}
\end{table*}

\subsection{Robustness analysis for the Clipping}
\label{apen:clipping}

As mentioned in Section \ref{subsec:comp}, we used the clipping for estimated labeling probability $\hat{u}({\bf x})$ to stabilize the training process.
Specifically, when the estimated labeling probability was smaller than a clipping value, we replaced it with the clipping value.
Here, we empirically investigated the robustness of the proposed method to the clipping values.
Table \ref{table:clipping} shows the average test AUCs by changing clipping value $\tau$ within $\{0.001, 0.005, 0.01, 0.05, 0.1\}$.
These results show that the proposed method is relatively robust against the choice of the clipping values.

\begin{table*}[t]
\caption{Average test AUCs when varying the clipping value $\tau$ for the estimated labeling probability $p(o=1|{\bf x})$.}
\label{table:clipping}
\centering
\scalebox{1.0}{
\begin{tabular}{lrrrrr}
\hline
Data & $\tau=0.001$ & $0.005$ & $0.01$ & $0.05$ & $0.1$ \\
\hline
Mnist    & 0.8945 & 0.8962 & 0.8979 & 0.8997 & 0.8995 \\
Fmnist   & 0.9467 & 0.9480 & 0.9495 & 0.9503 & 0.9494 \\
Svhn     & 0.6656 & 0.6672 & 0.6673 & 0.6525 & 0.6533 \\
Cifar10  & 0.8605 & 0.8624 & 0.8643 & 0.8642 & 0.8631 \\
Diabetes & 0.7325 & 0.7322 & 0.7342 & 0.7392 & 0.7400 \\
Blood    & 0.5480 & 0.5493 & 0.5509 & 0.5564 & 0.5564 \\
\hline
\end{tabular}
}
\end{table*}

\subsection{Results with Gaussian Noise on Estimated Labeling Probability}
\label{apend:results_gnoise_lp}

The proposed method uses the labeling probability estimated from the given PU data, $\hat{u}({\bf x})$, to construct the rewritten AUC estimator.
Here, we evaluated the proposed method when random noise was added to estimated labeling probability $\hat{u}({\bf x})$.
Specifically, we added the zero-mean Gaussian noise with standard deviations chosen from $\{0.01, 0.05, 0.1, 0.15 \}$ to $\hat{u}({\bf x})$.
When the modified labeling probability was less than $0.01$ or greater than $1$, we clipped the value to $0.01$ and $1$, respectively.
Table \ref{result_noisy_u} shows the results.
As expected, the performance of the proposed method tended to decrease as the noise increased.
However, the degradation was relatively small on several datasets. Even with a standard deviation of $0.15$, the proposed method often worked well. 
These results empirically demonstrate that the proposed method is
reasonably robust to estimation errors in the labeling probability.

\begin{table*}[t]
\caption{The results with noisy estimated labeling probabilities: average test AUCs over different positive class-prior $\pi$ within $\{0.05, 0.1, 0.15, 0.2\}$. Std represents the standard deviation of Gaussian noise. ${\rm Std}=0.0$ means that there is no noise.}
\label{result_noisy_u}
\centering
\scalebox{1.0}{
\begin{tabular}{lrrrrr}
\hline
Std & \multicolumn{1}{c}{0.0} & \multicolumn{1}{c}{0.01} & \multicolumn{1}{c}{0.05} & \multicolumn{1}{c}{0.1} & \multicolumn{1}{c}{0.15} \\
\hline
Mnist & 0.8979 & 0.8961 & 0.8933 & 0.8896 & 0.8867  \\
Fmnist & 0.9495 & 0.9501 & 0.9468 & 0.9403 & 0.9364 \\
Svhn & 0.6673 & 0.6595 & 0.6354 & 0.6214 & 0.6166 \\
Cifar10 & 0.8643 & 0.8689 & 0.8621 & 0.8567 & 0.8547 \\
Diabetes & 0.7342 & 0.7351 & 0.7325 & 0.7329 & 0.7328 \\
Blood & 0.5509 & 0.5453 & 0.5421 & 0.5412 & 0.5401 \\
\hline
\end{tabular}
}
\end{table*}

\subsection{Full Results}
\label{sec:full_result}

Table \ref{result_all_std} shows the average test AUCs with their standard deviations over different positive class-prior $\pi$ in each dataset. The proposed method performed the best or comparably to it in all cases.

\begin{table*}[t!]
\caption{Average test AUCs with their standard deviations over different positive class-prior $\pi$ within $\{0.05, 0.1, 0.15, 0.2\}$. 
We set $c=p(o=1|y=1)=0.1$. 
Values in bold are not statistically different at the $5\%$ level from the best performing method in each row according to a paired t-test.
}
\label{result_all_std}
\centering

\begin{subtable}{\textwidth}
\centering
\scalebox{0.95}{
\begin{tabular}{lrrrrr}
\hline
Data & Ours & w/oConf & NTC & nnPU & PUAUC \\
\hline
Mnist    & \bf{0.8979(0.0158)} & 0.8769(0.0226) & 0.8575(0.0192) & 0.8196(0.0531) & 0.8828(0.0203) \\
Fmnist   & \bf{0.9495(0.0236)} & 0.9371(0.0267) & 0.9165(0.0219) & 0.9296(0.0180) & 0.9409(0.0231) \\
Svhn     & \bf{0.6673(0.0688)} & \bf{0.6666(0.0579)} & 0.5328(0.0305) & 0.4885(0.0069) & \bf{0.6704(0.0539)} \\
Cifar10  & \bf{0.8643(0.0373)} & 0.8371(0.0420) & 0.8013(0.0491) & 0.4819(0.1191) & \bf{0.8620(0.0355)} \\
Diabetes & \bf{0.7342(0.0306)} & 0.6926(0.0268) & \bf{0.7357(0.0183)} & 0.5234(0.0775) & 0.7171(0.0258) \\
Blood    & \bf{0.5509(0.0317)} & 0.5287(0.0253) & 0.5409(0.0336) & 0.4847(0.0182) & 0.5395(0.0273) \\
\hline
\end{tabular}
}
\end{subtable}


\begin{subtable}{\textwidth}
\centering
\scalebox{0.95}{
\begin{tabular}{lrrrr}
\hline
Data & PUSB & PG & Pconf & NPU \\
\hline
Mnist    & 0.8841(0.0197) & 0.8575(0.0192) & 0.8193(0.0807) & 0.8285(0.0598) \\
Fmnist   & \bf{0.9474(0.0195)} & 0.9165(0.0219) & 0.7737(0.1616) & \bf{0.9470(0.0276)} \\
Svhn     & 0.5965(0.0763) & 0.5328(0.0305) & 0.4900(0.0070) & 0.4893(0.0068) \\
Cifar10  & 0.8442(0.0446) & 0.8013(0.0491) & 0.4767(0.1310) & 0.4883(0.1395) \\
Diabetes & 0.7230(0.0248) & \bf{0.7357(0.0183)} & 0.3974(0.0368) & 0.4626(0.0605) \\
Blood    & \bf{0.5450(0.0323)} & 0.5409(0.0336) & 0.4962(0.0204) & 0.4828(0.0200) \\
\hline
\end{tabular}
}
\end{subtable}

\end{table*}

\end{document}